\documentclass{article}

\usepackage{arxiv,times}
\iclrfinalcopy % Uncomment for camera-ready version, but NOT for submission.

\usepackage[utf8]{inputenc} % allow utf-8 input
\usepackage[T1]{fontenc}    % use 8-bit T1 fonts
\usepackage{hyperref}       % hyperlinks
\usepackage{url}            % simple URL typesetting
\usepackage{booktabs}       % professional-quality tables
\usepackage{amsfonts}       % blackboard math symbols
\usepackage{nicefrac}       % compact symbols for 1/2, etc.
\usepackage{microtype}      % microtypography
\usepackage{xcolor}         % colors
\usepackage{float}
\usepackage{placeins}

\usepackage{amsmath,amssymb}%,amsthm}
\usepackage{mathtools}
\usepackage{booktabs}
\usepackage{enumitem}
\usepackage{longtable}  

\newtheorem{assumption}{Assumption}
\newtheorem{remark}{Remark}
\newtheorem{lemma}{Lemma}
\newtheorem{definition}{Definition}

\newtheorem{proposition}{Proposition}
\newtheorem{theorem}{Theorem}
\newcounter{runex}
\newtheorem{runex}{Running Example}

\newcommand{\BlackBox}{\rule{1.5ex}{1.5ex}}
\usepackage{etoolbox}
\newenvironment{proof}{\par\noindent{\bf Proof\ }}{\hfill\BlackBox\vspace{2mm}}
\title{The Value of Mechanistic Priors\\ in Sequential Decision Making}

\newcommand{\Rmech}{R_{\mathrm{mech}}}
\newcommand{\Hmech}{H_{\mathrm{mech}}}
\newcommand{\Rtrain}{R_{\mathrm{mech}}^{\mathrm{train}}}
\newcommand{\Rtest}{R_{\mathrm{mech}}^{\mathrm{test}}}
\newcommand{\Bmu}{B_\mu}
\newcommand{\Bcrit}{\Bmu^{\mathrm{crit}}}
\newcommand{\lcrit}{l_\mu^{\mathrm{crit}}}
\newcommand{\pistar}{\pi^*}
\newcommand{\pihat}{\hat{\pi}}

\newcommand{\BR}{\mathrm{BR}}
\newcommand{\DKL}{D_{\mathrm{KL}}}
\newcommand{\norm}[1]{\left\|#1\right\|}
\newcommand{\E}{\mathbb{E}}
\newcommand{\R}{\mathbb{R}}
\newcommand{\kl}{\mathrm{kl}}
\newcommand{\argmax}{\arg \, \max}

\newcommand{\TShyb}{\mathrm{TS}_\mathrm{hyb}}
\newcommand{\TSuninf}{\mathrm{TS}_\mathrm{uninf}}
\newcommand{\muhyb}{\mathrm{\mu}_\mathrm{hyb}}
\author{%
  Itai Shufaro* \\
  Technion \\
  \texttt{itai.shufaro@campus.technion.ac.il} \\
  \And
  Gal Benor* \\
  Technion \\
  \texttt{gal.benor@campus.technion.ac.il} \\
  \AND
  Shie Mannor \\
  Technion, NVIDIA Research \\
  \texttt{shie@ee.technion.ac.il} \\
}

\begin{document}

\def\thefootnote{*}\footnotetext{These authors contributed equally.}

\maketitle

\begin{abstract}
Hybrid mechanistic models, physical priors with learned residuals, promise to reduce the
data required for good decisions, but have no computable criterion to test this. We
characterize the value of mechanistic priors in sequential decision-making within both
asymptotic and burn-in regimes. To formalize this, we introduce the \emph{mechanistic
information} of a model: the mutual information between the model's recommended policy
$\hat{\pi}$ and the true optimal policy $\pi^*$, bounded via a centered, occupancy-weighted
bias. In the \emph{asymptotic regime} (large $N$), matched bounds reveal that
Bayesian regret scales with the residual entropy $H_{\mathrm{mech}}$, delivering a theoretical sample
complexity reduction of $H(\mu)/H_{\mathrm{mech}}$ compared to an uninformed baseline. We further provide a
computable pre-trial model certificate. Complementarily,
in the clinically relevant \emph{burn-in regime} (small $N$), we establish a lower bound on
the penalty incurred by confidently wrong priors. We demonstrate both the asymptotic and
burn-in bounds on an illustrative in-silico 5-fluorouracil (5-FU) chemotherapy plant whose
structure follows published FOLFOX pharmacokinetics. The hybrid prior reduces cumulative regret by $1.79\times$
relative to standard body-surface-area dosing and $1.85\times$ relative to an uninformed
learner, and remains below both under every calibration bias tested. Finally, we show that priors sensitive to distribution
shift can lose half of their mechanistic information under a small distribution shift,
motivating physically grounded priors for safety-critical applications.

\end{abstract}

\section{Introduction}
\label{sec:intro}

Hybrid mechanistic models, known dynamics paired with learned residuals, are pervasive in scientific machine learning and increasingly central to adaptive medicine. Universal differential equations~\citep{rackauckas2020}, physics-informed neural networks~\citep{karniadakis2021}, grey-box pharmacokinetic models, and the digital twins underlying in silico clinical trials~\citep{eadie2026} all share one structural premise: physical structure reduces the data required to reach good decisions. The premise is taken for granted, yet no computable, pre-trial criterion tells a decision-maker whether a specific hybrid model will save samples on a specific sequential task. That criterion is the goal of this paper.

A main obstacle is that the quantity routinely reported for a hybrid model, its predictive error on held-out observations, is not the quantity that governs its value for decisions. Predictive accuracy is neither necessary nor sufficient for choosing well: a model may fit accurately and still order two adjacent actions incorrectly, while a model that is badly wrong where no candidate policy operates, or wrong by an offset shared by every action, may still recommend the optimal one. What governs the value of a model is the error it commits where the decision is actually made, measured in a form that can be converted into interactions saved.

Information-theoretic bounds linking prior information to sample complexity exist for other online settings. \citet{russo2016} show that the Bayesian regret of Thompson sampling on a $K$-armed bandit scales with the entropy of the optimal arm, and \citet{shufaro2025bits} extend this to a regret--information trade-off that quantifies the value of $R$ nats of side information. Both treat the prior as exogenous: $R$ nats are assumed to be available, and the analysis prices them. In a hybrid model the prior is constructed rather than assumed, and the nats it carries are fixed by the mechanistic model's structure, by the data used to calibrate it, and by the error it retains.

Mechanistic models are treated here as priors for sequential learners, and their worth is measured by the information they carry about the optimal policy. We introduce the \emph{mechanistic information} of a model, the number of nats it reveals before any interaction takes place, and establish that the regret still to be incurred is governed by what the model leaves unrevealed. Lower and upper bounds on this regret are derived. Two regimes are covered: the asymptotic one, where the prior's worth appears as a reduction in sample complexity, and the finite-horizon one, where a confident but incorrect prior imposes a cost that must be paid before it can be overturned, the situation that arises whenever a population model is applied to an individual. From this follows a threshold on model error that can be evaluated before experimentation begins, allowing a practitioner to determine in advance whether a candidate mechanistic model is accurate enough to be worth deploying within a hybrid learner. Figure~\ref{fig:setting} summarizes the setting, and adaptive 5-fluorouracil (5-FU) dosing serves as an illustration throughout.

\subsection{Contributions}
\textbf{Mechanistic information and critical bias}
We define \emph{mechanistic information} $\Rmech = I(\pistar; \pihat)$, the mutual information between the model's recommended policy $\pihat$ and the optimal policy $\pistar$, and bound it through an occupancy-weighted bias $\Bmu$ that charges the model only for errors that affect decision-making. From this bound, we derive the \emph{critical bias} $\Bmu^{\mathrm{crit}}$: a closed-form quantity whose inputs are the population spread of the measured parameter and the prior's entropy, both of which can be obtained from reported population-pharmacokinetic studies. This bias gives a \emph{necessary} condition, checkable before any interaction, for a candidate model to reduce sample complexity by a target factor $s$.
% We introduce \emph{mechanistic information} ${\Rmech = I(\pistar;\pihat)}$---the mutual information between the model's recommended policy $\pihat$ and the true optimal policy $\pistar$. We then introduce upper bounds on $\Rmech$ and leverage them to introduce the \textit{critical bias} $\Bcrit$, which measures a model's ability to reduce sample complexity and does not require additional trials to compute. 
%NEW: INSTEAD OF ASYMPTOTIC:

\textbf{Matched regret bounds}
Theorems~\ref{thm:lb} and~\ref{thm:ub} apply the bounds of \citet{shufaro2025bits} and \citet{russo2016} to the \emph{residual entropy} $\Hmech = H(\mu) - \Rmech$, giving matching Bayesian-regret bounds (tight up to $\sqrt{\log K}$) and a sample-complexity reduction of $H(\mu)/\Hmech$ over an uninformed baseline; the new ingredients are the measurement-channel bound of Proposition~\ref{prop:mechinf}, the centered occupancy-weighted bias $\Bmu$, the rule-out certificate of Theorem~\ref{thm:phase}, and the burn-in bound of Proposition~\ref{prop:burnin}. These bounds allow $\Bcrit$ to turn from an exact information-level threshold to a sample-count target. 

% \paragraph{Asymptotic bounds} 
% To demonstrate the utility of the critical bias, we extend the framework of \citet{shufaro2025bits} to mechanistic models and introduce regret upper and lower bounds. These bounds directly quantify the relationship between $\Rmech$, $\Bcrit$ and the samples avoided using the model.
\textbf{Burn-in regime} We lower-bound the regret suffered when the prior is concentrated but potentially wrong, the regime of a population-calibrated model deployed for an individual patient. We also show that priors fitted to a population distribution, can lose most of their mechanistic information under small distribution shifts, motivating auditable, physically grounded priors in safety-critical settings.

\textbf{Paper organization.}
Section~\ref{sec:setting} reduces the problem to a policy bandit and introduces mechanistic information; Section~\ref{sec:main} proves the asymptotic results and the model-quality certificate; Section~\ref{sec:miss_prior} treats misspecified priors; Section~\ref{sec:application} is the clinical instantiation and simulation.
\section{Setting}
This section formalizes the reduction from the continuous control problem to a $K$-armed policy bandit and introduces the central quantity $\Rmech$. This frames our problem as repeated selection from a discretized set of dose policies: each round is one treatment cycle and each arm is one dose level held for that cycle.
\label{sec:setting}
\begin{figure}[b]
\centering
\includegraphics[width=0.9\linewidth]{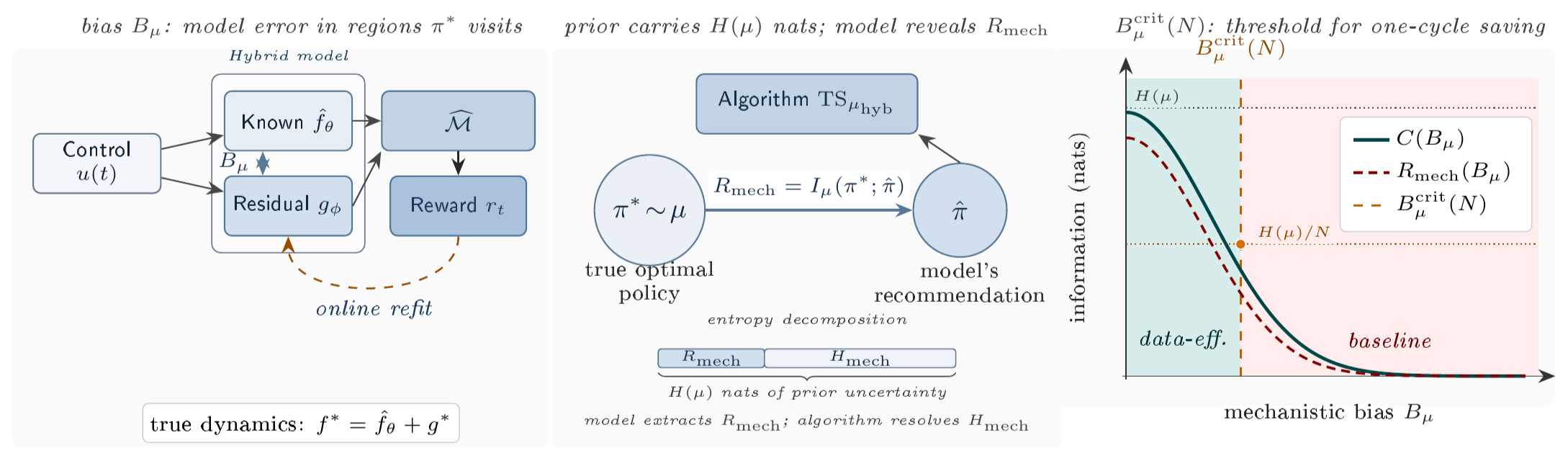}
\caption{Three layers of the hybrid model learning setting.
  \textbf{(A)} Mechanistic model.
  \textbf{(B)} Information flow from $\pistar$ to
  $\pihat$ and to the environment.
  \textbf{(C)} Regret-bias diagram: switching regimes at $\Bmu^{\mathrm{crit}}$.
  }
\label{fig:setting}
\end{figure}
\subsection{Reduction to policy bandit}
\label{sec:ode_to_bandit}
The mechanistic model $\mathcal{M}$ defines a distribution over the dynamics of the continuous control problem. Given a control signal $u(t)$ and a state trajectory $x(t)$, the cumulative reward is $J(u;\,{\cal{M}}) = \int_0^T \ell(x(t),u(t))\,dt$. The true unknown dynamics $\cal{M}^*$ are drawn from a prior $\rho$ over models. The continuous optimum is $\pi^{\mathrm{opt}} = \argmax_{u\in\mathcal{U}} J(u;\,\cal{M}^*)$. We discretize the control space into $K$ representative policies, $\Pi = \{\pi_1,\ldots,\pi_K\}$ (e.g.\ via a finite-difference grid on the infusion rate), reducing our control problem to a K-MAB one. $\pistar := \argmax_{\pi\in\Pi} J(\pi;\,{\cal{M}}^*)$ is the best discrete policy, which is the object the bandit learns. Each policy $\pi_k$ has mean reward $\bar{r}(\pi_k) = J(\pi_k;\cal{M}^*)$, At round $t$ the learner plays $\pi_t \in \Pi$ and observes $r_t = \bar{r}(\pi_t) + \xi_t$, where $\xi_t$ is normally distributed, with zero-mean and standard deviation $\sigma_\xi>0$. We make the following assumption on the quality of the discretization. 
\begin{assumption}[Discretization quality]
\label{ass:disc}
The policy $\pi^{\mathrm{opt}}$ belongs to $\Pi$, or the best policy in $\Pi$ has sub-optimality gap at most $\Delta_r/2$ relative to the continuous optimum,
\begin{equation}
    J(\pi^{\mathrm{opt}};{\cal{M}^*}) - J(\pistar;{\cal{M}^*}) \leq \frac{\Delta_r}{2}.
\end{equation}
\end{assumption}

Before interaction starts, an observation is drawn $Z = h(M^*) + \nu$ where $ \nu \sim \mathcal{N}(0, \sigma_\nu^2 I)$ independent of $M^*$, $\dim Z = d_Z$ and $h : \mathrm{supp}(\rho) \to \mathbb{R}^{d_Z}$ is an invertible map. The calibrated model is the model consistent with the measurement, $\hat{\cal{M}} = h^{-1}(Z)$, where $h^{-1}$ is the inverse on decision-relevant equivalence classes guaranteed by Assumption~\ref{ass:h_llb}. The learner then plays under $\hat{\cal{M}}$, which recommends the best policy in $\Pi$ under the approximate dynamics. 

\begin{equation}
  \pihat \;=\; \argmax_{\pi\in\Pi}\, J(\pi;\,\hat{{\cal{M}}}).
  \label{eq:pihat}
\end{equation}

This policy is computable from the model $\hat{\cal{M}}$ alone, without knowing the true dynamics. We remark that $\pihat$ is not the policy for which the model is most accurate everywhere, but rather it is the policy that performs best under the model's predictions. Even an inaccurate model can produce a useful $\pihat$ if its prediction errors cancel in the reward integral. We denote by $\mu$ the marginal distribution of $\pi^*$, when $\mathcal{M}^* \sim \rho$. We write $H(\mu)$ for the Shannon entropy of the prior $\mu$ over $\Pi$. All logarithms are natural (nats) unless explicitly labeled $\log_2$ (bits). We use the Bayesian regret, which is the mean gap between the optimal and chosen rewards over the randomness present in the prior, $\pi^*$, and the algorithm. This metric measures the prior's worth in terms of interactions saved.
\begin{definition}[Bayesian regret]
\label{def:br}
\[
  \BR^*(N) = \E_{\pistar\sim\mu}\E\!\left[
    \sum_{t=1}^N \bigl(\bar{r}(\pistar) - \bar{r}(\pi_t)\bigr)
  \right].
\]
\end{definition}
We write $\sigma_h^2 := \mathrm{tr}\,\mathrm{Cov}_{M^*\sim\rho}\!\left(h(M^*)\right)$ for the population spread of the measured quantity. 

\subsection{Occupancy-weighted sensitivity}
\label{sec:occ_sensitivity}
\begin{assumption}[Metric model space] \label{ass:metric}
    The model space is metric, with the centered occupancy-weighted reward norm
    \[ \norm{{\cal{M}} - {\cal{M}}'}_{\mu}^{2} = \sum_{k=1}^{K} \mu(\pi_k)\Big(\big(J(\pi_k;{\cal{M}}) - \bar{J}({\cal{M}})\big) - \big(J(\pi_k;{\cal{M}}') - \bar{J}({\cal{M}}')\big)\Big)^{2} \]
    when $\bar{J}({\cal{M}}) = \sum_{k}\mu(\pi_k)J(\pi_k;{\cal{M}})$.
\end{assumption}
\begin{assumption}[Identifiability of the decision-relevant model] \label{ass:h_llb}
    There exists $L_h > 0$ such that for all ${\cal{M}}, {\cal{M}}'$ in the support of $\rho$,
    \[ \sqrt{d_Z}L_h \norm{{\cal{M}} - {\cal{M}}'}_{\mu} \;\leq\; \norm{h({\cal{M}}) - h({\cal{M}}')}_{2}, \]
    where $\norm{\cdot}_{\mu}$ is the norm of Assumption~\ref{ass:metric}.
\end{assumption}
Appendix~\ref{sec:app_ellh} computes $L_h$ for the 5-FU model; notation is collected in Appendix~\ref{app:notation}.
The occupancy weighting charges the model only for errors in regions $\pistar$
actually visits. Equation~\eqref{eq:bias} is the expectation of this squared norm between
the true and the calibrated model. Concretely, 
\begin{equation} \label{eq:bias}
\Bmu^2 = \E_{M^*\sim\rho,\,\nu}\!\norm{{\cal{M}}^*-h^{-1}(Z)}_\mu^2
\end{equation}
This bias is the expected centered, occupancy-weighted squared reward gap between the true and the calibrated model; a common offset of the reward vector scores zero because it does not change the recommendation. A model that is inaccurate in corners of the state space that the policies in $\Pi$ never reach has a small $\Bmu$, while a model that is slightly wrong exactly where they operate has a large $\Bmu$ and transfers almost none. Throughout this paper, we will denote $l_\mu = B_\mu L_h$.

 Throughout this paper, we use a running example, for illustration, based on adaptive dosing of 5-fluorouracil (5-FU) in colorectal cancer patients receiving FOLFOX.
\begin{runex}[Adaptive 5-FU dosing]
\upshape
\label{ex:5fu_setup}
We instantiate the true dynamics ${\cal{M}}^*$ with the population pharmacokinetic model of \citet{tan2026advancing}, a two-compartment ODE with Michaelis--Menten elimination whose state comprises the 5-FU amounts in the central and peripheral compartments. The deployed mechanistic model $\hat{\cal{M}}$ is a one-compartment linear-clearance \emph{surrogate} of the Tan model, built from the pre-trial measurement $Z$ of Section~\ref{sec:ode_to_bandit}, a single plasma concentration drawn during a body-surface-area--dosed lead-in cycle. The policy class $\Pi=\{\pi_1,\ldots,\pi_8\}$ ($K=8$) represents eight dose levels spanning 1{,}600--3{,}600~mg/m$^2$, the clinical adjustment range in the \citet{kaldate2012} dataset. Reward $r_t\in\{0,1\}$ indicates whether the realised 46-hour AUC falls within the therapeutic window $20$--$30$ mg$\cdot$h/L.
\end{runex}

%%%%%%%%%%%%%%%%%%%%%%%%%%%%%%%%%%%%%%%%%%%%%%%%%%%%
\subsection{Mechanistic Information}
\label{sec:mechinf}
%%%%%%%%%%%%%%%%%%%%%%%%%%%%%%%%%%%%%%%%%%%%%%%%%%%%
\begin{definition}[Mechanistic information]
\label{def:mechinf}
\[
  \Rmech \;=\; I_\mu(\pistar;\,\pihat),
\]
the mutual information between $\pistar\sim\mu$ and the mechanistic
recommendation $\pihat$ defined by~\eqref{eq:pihat}.
\end{definition}
$\Rmech$ measures the number of nats about $\pistar$ that the model reveals before a single interaction is performed.  When $\Bmu=0$ the model is exact ($\pihat = \pistar$, $\Rmech = H(\mu)$); as $\Bmu\to\infty$ the recommendation becomes independent of $\pistar$ and $\Rmech\to 0$. The residual entropy $\Hmech = H(\mu) - \Rmech$ quantifies the information the algorithm must discover by interacting with the system. We now introduce an upper bound on the mechanistic information, equivalently a lower bound on the residual entropy.

\begin{proposition}[Mechanistic information bounds]
\label{prop:mechinf}
Under Assumptions ~\ref{ass:metric}-\ref{ass:h_llb},
\begin{align}
  \Rmech
  \;\leq\; C(B_\mu) \;:=\; \min\Big\{\, H(\mu),\;
    \frac{d_Z}{2} \log\Big(1+\frac{\sigma_{h}^2}{L_h^2B_\mu^2}\Big) \Big\}.
    \label{eq:mechinf_ub}
\end{align}
\end{proposition}
The proof is provided in Appendix \ref{sec:app_setting}. The mechanistic model essentially acts as a noisy Gaussian channel. The model bias contributes noise variance $l_\mu^2$, so the SNR, and with it the information the model can transmit, decreases as $\Bmu$ grows. Appendix~\ref{sec:appCalibration_certificate} provides further details regarding this Proposition.

\begin{runex}[Mechanistic information for 5-FU]
\upshape
\label{ex:5fu_mechinf}
The certificate inputs are two constants \emph{derived} from published population-PK statistics (\citet{tan2026advancing}, Appendix~\ref{sec:appCalibration}): the between-patient spread of the elimination parameter, $\sigma_h = \sqrt{0.10}$, and the precision of its pre-dose estimate from a three-sample limited-sampling scheme. We get that (see Appendix~\ref{sec:appCalibration})
\[
C(\Bmu) = 0.77.
\]
Appendix ~\ref{sec:appCalibration} provides full details regarding this calculation and the parameter's calculation. 
\end{runex}

\section{Asymptotic Regime}
\label{sec:main}
\subsection{Lower and Upper bounds}
We now lower-bound the Bayesian regret via the residual entropy $\Hmech$ and prove a matching upper bound up to a $\sqrt{\log K}$ factor.
\begin{theorem}[Hybrid regret lower bound]
\label{thm:lb}
Consider the hybrid setting with $K$ policies and $N$ rounds.  Under Assumptions \ref{ass:metric}--\ref{ass:h_llb}, there exist universal constants $c,N_0>0$ such that for any algorithm there exists a prior $\mu$ over $\Pi$ such that for all $N\geq N_0$:
\begin{equation}
  \BR^*(N) \;\geq\; c\,\sqrt{\frac{KN\,\Hmech}{\log K}},
  \label{eq:lb}
\end{equation}
where $\Hmech = H(\mu)-\Rmech \geq 0$.
\end{theorem}

Theorem \ref{thm:lb} quantifies the regret the model saves: the $\Rmech$ nats it ``spends'' before interaction reduce the effective problem from $H(\mu)$ to $\Hmech$ nats of uncertainty. Thus, the ratio $\sqrt{H(\mu)/\Hmech}$ quantifies the reduction in the regret that is obtained by following the model's recommendation. The proof of Theorem \ref{thm:lb} follows directly from Proposition 4.1 of \citet{shufaro2025bits}, and is provided in Appendix \ref{app:main_results}.

\begin{theorem}[Thompson Sampling upper bound]
\label{thm:ub}
Let $\mu_{\mathrm{hyb}}$ be the posterior of $\mu$ after conditioning on $\pihat$. By construction, $I(\pistar;\pihat) = \Rmech$. Let $\mathrm{TS}_{\mathrm{hyb}}$ be Thompson Sampling with this prior.  Then there exists a universal constant $C>0$ such that:
\begin{equation}
  \BR_{\TShyb}(N) \;\leq\; C\sqrt{KN\,\Hmech},
  \label{eq:ub}
\end{equation}
% where \begin{equation}
%     H(\mu_\mathrm{hyb})=\log(e^{\Rmech}+K-1)-\Rmech e^{\Rmech}/(e^{\Rmech}+K-1).
%     \label{eq:H_hyp_exact_thm}
% \end{equation}
% For $K\gg e^{\Rmech}$, $H(\mu_\mathrm{hyb})=\Hmech+O(e^{\Rmech}/K)$
% and the upper bound in Equation ~\eqref{eq:ub} becomes $C\sqrt{KN\Hmech\log K}$.
\end{theorem}
Theorem \ref{thm:ub} shows that Thompson Sampling with $\mu_\mathrm{hyb}$ achieves the lower bound: when the prior is concentrated on $\pihat$, the algorithm's first prediction identifies $\pihat$ with probability $\mu_\mathrm{hyb}(\pihat)$ (close to 1 when $\Rmech \approx \log K$), and subsequent rounds refine or overturn this guess. The proof of Theorem \ref{thm:ub} follows from \citep{russo2016} and can be found in Appendix \ref{app:main_results}.

We can define the sample complexity algorithm as the average number of rounds required to achieve mean regret of less than $\varepsilon$, so 
\[
N(\varepsilon) = \min_{N > 0 : \frac{1}{N}\mathcal{BR}(N) \leq \varepsilon} N
\]
Combining Equations~\eqref{eq:lb} and~\eqref{eq:ub}, the number of rounds needed to reach mean regret $\varepsilon$, the sample complexity, obeys the asymptotic ratio:
\begin{equation}
  \rho(\varepsilon) \;=\; \frac{N_{\text{uninf}}(\varepsilon)}{N_{\text{mech}}(\varepsilon)}
         \;=\; \Theta\left(\frac{H(\mu)}{\Hmech}\right).
  \label{eq:ratio_full}
\end{equation}
where the $\Theta$ notation omits factors of $\Theta(\log K)$.
\begin{runex}[Lower and upper bounds for 5-FU]
\label{ex:5fu_bounds}
\upshape
We apply Theorems~1 and~2 at the channel ceiling of Running Example~\ref{ex:5fu_mechinf}: a prior that attains its channel ceiling carries $\Rmech = 0.77$ nats, i.e.\ $\Hmech = 1.31$ nats. With $K = 8$ and $N = 12$ cycles, $7.77c \leq \mathrm{BR}^*(12) \leq 11.20C$; the ratio $11.20/7.77 = 1.44 = \sqrt{\ln 8}$ is exactly the $\sqrt{\log K}$ slack between the two bounds. The sample-complexity ratio (Eq.~(\ref{eq:ratio_full})) is certified at $\rho = H(\mu)/\Hmech = 2.08/1.31 \approx 1.59$ (up to $\log K$ factors): an uninformed learner needs up to $1.59\times$ more mis-dosed cycles to reach the same regret (a ratio of residual entropies; the samples reading inherits the constants of Eq.~\eqref{eq:ratio_full}).
\end{runex}

\subsection{Critical Bias}
The previous section demonstrated how one can use information to effectively quantify regret improvement. However, to be more relevant in clinical settings, we need a quantity that is both relevant to the number of trials we realistically reduce, while also being computable without requiring any additional trials. This section aims to achieve that. We now introduce the critical bias, a computable model certificate that governs the model's ability to reduce sample complexity. 
\begin{theorem}[Critical bias, rule-out form]
\label{thm:phase}
Let $s \geq 1$ be a target sample-complexity ratio. For $s > 1$ define the critical bias,
\begin{equation}
  L_h\Bcrit(s) = {l_\mu^{\mathrm{crit}}}(s) = \frac{\sigma_h }{\sqrt{\exp\left(\frac{2}{d_Z} (1- 1 / s) H(\mu)\right) - 1}}.
  \label{eq:bcrit}
\end{equation}
Then any model with $B_\mu \geq \Bcrit(s)$ suffers regret of at least $c\sqrt{KNH(\mu)/(s\log K)}$ for any $N \geq N_0$ (where $N_0$ is defined in Theorem \ref{thm:lb}).
\end{theorem}
Theorem~\ref{thm:phase} gives a \emph{necessary} condition, checkable before any interaction, for a candidate model to reduce sample complexity by a target factor $s$: a model with bias above $\lcrit(s)$ cannot, whatever the algorithm. A model below it may, depending on how close $\Rmech$ comes to the capacity bound. The threshold is a statement about residual decision entropy; reading it as a factor-$s$ saving in samples goes through Eq.~\eqref{eq:ratio_full} and inherits its constants. The one-cycle-of-$N$ reading corresponds to $s = N/(N-1)$. $\lcrit(s)$ is computable before any trial. Locating a model against it needs that model's $l_\mu$ (Limitations). The proof of Theorem \ref{thm:phase} is provided in Appendix \ref{app:main_results}.
\begin{remark}
    The choice $s = N/(N - 1)$ is illustrative. Since Theorems 1–2 match only up to constants and a $\sqrt{\log K}$ factor, a one-cycle difference lies below the resolution of the bounds. $\Bcrit$ at this point should therefore be read as an information-level screen: a model above it cannot carry enough mechanistic information to meet the target, while passing it is necessary but not sufficient for any saving in cycles.
\end{remark}
\begin{runex}[Model-quality certificate for 5-FU]
\upshape
With the Tan-sourced constants ($\sigma_h = 0.316$, $L_h = 1$, $d_Z = 1$, $H(\mu) = \ln 8$) and $N = 12$, the critical bias~\eqref{eq:bcrit} at the one-cycle working point $s = N/(N-1) = 12/11$ (i.e.\ $(1-1/s)H(\mu) = H(\mu)/12$) is $\Bcrit = 0.316/\sqrt{8^{1/6}-1} \approx 0.49$. This is the working point the clinic can act on: can the model be expected to save one cycle out of $N$? 
\end{runex}

\section{Misspecified priors in the finite-sample regime}
\label{sec:miss_prior}
Section~\ref{sec:main} characterized regret in the asymptotic regime. We now turn to the finite-sample regime relevant to clinical-trial design.
\begin{proposition}[Burn-in lower bound]
\label{prop:burnin}
Let $K \geq 2$, $\Delta \in (0,1)$ and $\delta \in (0, 1)$. Let the prior place mass
$1-\epsilon$ on $\pihat$ (confident but wrong) and $\epsilon/(K-1)$ on each other arm, one of
which is $\pistar$, so that the prior log-odds are $L_0 = \log\!\big((1-\epsilon)(K-1)/\epsilon\big)$.
There is a Bernoulli bandit with gap $\Delta$ such that any algorithm for which
$\Pr(\pi^*=j \mid \mathcal{F}_t) \ge \Pr(\pi^*=k \mid \mathcal{F}_t)$ implies
$\Pr(A_{t+1}=j \mid \mathcal{F}_t) \ge \Pr(A_{t+1}=k \mid \mathcal{F}_t)$
then $\tau_\delta$ satisfies
\begin{equation}
  \E\big[\text{regret up to } \tau_\delta\big] \;\geq\;
  \frac{(1-\delta)\,L_0 + 2\log\frac{1}{2.4\delta} - 1}{4\,\log\frac{1+\Delta}{1-\Delta}}.
  \label{eq:burnin}
\end{equation}
\end{proposition}
Proposition \ref{prop:burnin} quantifies the minimum exploration cost that a confidently wrong prior forces. In particular, it grows with the prior log-odds $L_0 = \log[(1-\epsilon)(K-1)/\epsilon]$ and with the identification requirement through $\log(1/(2.4\delta))$, and for small $\Delta$ it scales as $1/\Delta$ (the standard rate, since $\log\frac{1+\Delta}{1-\Delta} \approx 2\Delta$). The cost grows logarithmically as the prior becomes more confident ($\epsilon\to0$) or as the identification confidence requirement tightens ($\delta\to0$). The proof of Proposition \ref{prop:burnin} can be found in Appendix \ref{sec:app_finite_sample}.

\begin{runex}[Burn-in for a miscalibrated 5-FU prior]
\upshape
Using the $K=8$ setting: the prior places mass $1-\epsilon = 0.8$ on $\pihat$ (a confident but, for this patient, wrong recommendation), so $L_0 = \ln\!\big(0.8\times 7/0.2\big) = \ln 28 \approx 3.33$; the sub-optimality gap is $\Delta = 0.2$ (approximately one dose step) and the identification confidence is $\delta = 0.01$. Substituting into \eqref{eq:burnin},
\[
  \text{Burn-in} \;\geq\;
  \frac{0.99\times 3.33 + 2\ln\frac{1}{0.024} - 1}{4\,\ln(1.2/0.8)} \cdot
 \approx 6.02,
\]
which is equivalent to $\approx 30$ plays of a $\Delta$-suboptimal arm (relaxing $\delta$ to $0.10$ lowers the bound to $\approx 15$ suboptimal plays). The cost of a confidently wrong prior is computable in closed form before any interaction, from the prior weight, the dose-step gap, and the required identification confidence alone.
\end{runex}

The following proposition demonstrates how misspecified priors limit mechanistic information.
\begin{proposition}[Misspecified prior information bounds]
\label{prop:llm}
Let some policy produce recommendations with mechanistic information $\Rtrain$ on training distribution $P_{\mathrm{train}}$, and assume the recommender is symmetric: $H(\pihat \mid \pistar = k)$ is equal for all $k$ (for example, a symmetric confusion channel). Let the deployment distribution $P_{\mathrm{test}}$ have KL shift
$\Delta_\pi=\DKL(P_{\mathrm{train}}\|P_{\mathrm{test}})$. 
Denote by $\Rtest$ the mechanistic information of the recommendations under the test distribution. Then there exist train and test distributions with
$\Delta_\pi = \frac12\Rtrain$, with uniform $\hat{\pi}$ for which $\Rtest\leq\frac12\Rtrain + \frac12\log K$.
\end{proposition}
The proof of Proposition \ref{prop:llm} is provided in Appendix \ref{sec:app_finite_sample}. A small shift, undetectable from training data alone, can thus significantly degrade performance --- a patient who reacts very differently from what the training literature describes. The mechanistic prior is not immune to shift, since a population model without DPD status as an input is as wrong for a DPD-deficient patient as any population-trained prior, which is the situation Proposition~\ref{prop:burnin} describes. Its advantage is {auditability}. $\sigma_h$ and $\Bmu$ can be re-estimated on any cohort, so the bound of Proposition~\ref{prop:mechinf} can be re-certified after a shift rather than assumed.
\begin{runex}[Misspecified priors for 5-FU]
\upshape
A distribution calibrated on a portion of FOLFOX population data might have $\Rtrain=1.6$~nats. The calibration data might suffer a shift of $\Delta_\pi\approx0.5$~nats (rough estimate). By Proposition~\ref{prop:llm}, test distributions with $\Delta_\pi\leq\frac12\log K \approx 1.04$~nats exist for which $\Rtest\leq\frac12\log K$. When $P_{\pihat}$ is uniform, a shift of $\Delta_\pi = \frac12\Rtrain$ already halves the information, and a shift of this size can be undetectable from training data alone. ODE priors are often physically bounded, allowing them to provide information while remaining robust to distribution shifts. Calibrated LLM priors, on the other hand, might overfit and fail on parts of the population while appearing to provide more information.
\end{runex}

\begin{runex}[End-to-end, in silico]
\upshape
\label{ex:5fu_endtoend}
We close the loop on a simulated plant: the world is the two-compartment model of \citet{tan2026advancing}. The deployed prior is the one-compartment surrogate of Running Example~\ref{ex:5fu_setup}, fitted on a population slice and inverted once on a lead-in blood draw. Measured on this plant, the draw takes $\Rmech$ from $0.000$ to $0.42$ nats, below its certified ceiling $C(\Bmu) = 0.77$, and Thompson sampling with the resulting prior, the algorithm of Theorem~\ref{thm:ub} observing rewards alone, cuts cumulative regret by $1.8\times$ against both BSA dosing and uninformed TS. Its regret is also comparatively stable: across calibration slices biased in either direction, it degrades by about a tenth at the harshest slice ($1.66$ to $1.83$) while staying below both model-free references, whereas the per-cycle ODE controller varies more than five-fold and ends worse than standard-of-care dosing (Table~\ref{tab:appCalibration_hyb_ode}). Appendix~\ref{sec:appCalibration}describes the full protocol and constants. 
\end{runex}

\section{Application and Simulation}
\label{sec:application}
\subsection{Clinical context: the 5-FU individualisation problem}
5-Fluorouracil (5-FU) is the backbone of FOLFOX chemotherapy for colorectal cancer, yet only roughly twenty percent of body-surface-area (BSA)-dosed patients reach the therapeutic exposure window~\citep{li2023,saif2009}. We instantiate the framework on literature-calibrated 5-FU pharmacokinetic parameters. \footnote{Code for the experiments is available here: \url{https://github.com/galbenor/The-Value-of-Mechanistic-Priors-in-Sequential-Decision-Making}}

\paragraph{Literature-calibrated parameters}
We calibrate the bandit parameters to published population PK data. The policy class $\Pi$ is $K = 8$ discrete dose levels spanning the clinical adjustment range, $1{,}600$--$3{,}600$~mg/m$^2$ \citep{kaldate2012}. Arms above a patient's DPD-adjusted ceiling, known before cycle~1, would be removed from $\Pi$ first in a clinical instantiation, a gate we do not model. Rewards are normalized: $r_t = 1$ if the AUC falls within $20$--$30$~mg$\cdot$h/L, $r_t = 0$ otherwise, with Bernoulli noise. The reward is symmetric in under- and over-exposure, a simplification, since clinically overexposure drives grade~3/4 toxicity. Under standard BSA dosing, the target attainment probability is $p_{\mathrm{BSA}} \approx 0.20$ \citep{li2023} and since $r_t$ is Bernoulli, $\sigma_\xi = \sqrt{p(1-p)} \approx 0.40$. The optimal (PK-guided) policy achieves $p_{\mathrm{opt}} \approx 0.85$, an idealized \emph{per-cycle} attainment at the individually optimal dose, where PK-guided \emph{titration} reached the range in $94\%$ of patients only cumulatively~\citep{gamelin2008multivariable}. Appendix~\ref{sec:appCalibration} sweeps these choices. The certificate inputs and their ceiling ($C \approx 0.77$ nats, $H_{\mathrm{mech}} \geq 1.31$) are those of Running Example~\ref{ex:5fu_mechinf}, and every operating point below respects that ceiling.

\subsection{Simulation results}
\label{sec:sim}
\setlength{\intextsep}{5pt}
The plant is the two-compartment Michaelis--Menten model of \citet{tan2026advancing}, its per-patient parameters drawn from an illustrative population whose magnitudes are rounded modeling choices rather than estimates of any cohort (Appendix~\ref{sec:appCalibration_reward}). 
The deployed model is the
one-compartment surrogate of Running Example~\ref{ex:5fu_setup}, fitted on a $2{,}000$-patient calibration study and inverted once on a lead-in blood draw. Evaluation uses a disjoint cohort of $n = 10{,}000$, $K = 8$ arms, one unscored lead-in cycle then $11$ scored cycles, all policies on the same patients. The hybrid implements Theorem~\ref{thm:ub} by estimating the posterior. Its prior $\muhyb = P(\pistar \mid \pihat)$ is frozen after the draw, and the binary reward is its only later input (Appendix~\ref{sec:appCalibration}).

We compare TS with the hybrid prior ($\TShyb$) against three different policies: uninformed TS ($\TSuninf$), a static BSA baseline, and a policy that sticks with the ODE recommendation (ODE only). We use the first reward as the initial observation ($Z$). The ODE-only policy uses this information to calibrate its parameters but does not use later observations. Table \ref{tab:rmech} presents the comparison between these four policies for the calibrated values previously described. For more details, see Appendix~\ref{sec:appCalibration}. We also run a full parameter sweep in Appendix \ref{sec:appCalibration_channelsweep}. We see that the mechanistic information provides TS a nearly $2\times$ improvement over the uninformed baseline, which is more than the predicted improvement of $\approx 1.12$ (using the lower bound ratio $\sqrt{H(\mu) / \Hmech}$.

\begin{table}[h]
\centering
\footnotesize
\setlength{\abovecaptionskip}{3pt}
\setlength{\belowcaptionskip}{1pt}
\setlength{\tabcolsep}{2.5pt}
\caption{The calibrated plant, $n = 10{,}000$, seeds 11/2024: cumulative regret over the
$11$ scored cycles (mean $\pm$ 95\% CI), the comparator/hybrid ratio of means with
paired-bootstrap intervals, what each policy reads in a scored cycle and how often it
solves the ODE. The lead-in draw is taken from everyone. Per-cycle controller:
Table~\ref{tab:appCalibration_hyb_ode}.}
\label{tab:rmech}
\begin{tabular}{lccccc}
\toprule
policy & regret & vs.\ hybrid & lead-in draw & per-cycle & solves \\
\midrule
uninformed TS & $3.07 \pm 0.02$ & $1.85\times$ [1.82, 1.88] & not read & binary & 0 \\
BSA static & $2.97 \pm 0.04$ & $1.79\times$ [1.76, 1.83] & not read & none & 0 \\
ODE, one solve (prior only) & $1.83 \pm 0.04$ & $1.11\times$ [1.08, 1.13] & read & none & 1 \\
hybrid (Thm.~\ref{thm:ub}: TS, ODE prior) & $1.66 \pm 0.02$ & $1.00\times$ & read & binary & 1 \\
\bottomrule
\end{tabular}
\end{table}

Table~\ref{tab:channel} generates $\Rmech$ scenarios from the calibrated model using the pre-dose channel and assay precision $\sigma_{\mathrm{assay}}$, not from model bias. At $\Rmech = 0$ the model recommends one arm to everyone, so a learner ignoring the draw is uninformed TS.

\begin{table}[h]
\centering
\footnotesize
\setlength{\abovecaptionskip}{3pt}
\setlength{\belowcaptionskip}{1pt}
\renewcommand{\arraystretch}{0.95}
\setlength{\tabcolsep}{1.8pt}
\caption{Mechanistic information generated by the \emph{channel} at the calibrated model
($\gamma = 1$, $11$ scored cycles, $n = 10{,}000$), floor $= \sqrt{H(\mu)/\Hmech}$,
``ODE, once'' the one-solve policy of Table~\ref{tab:rmech}, uninformed TS $3.07$ and BSA
$2.96$ constant throughout.}
\label{tab:channel}
\begin{tabular}{cccccccccc}
\toprule
$\sigma_{\mathrm{assay}}$ & $\Rmech$ & $C(\Bmu)$ & $\Bmu$ & floor &
  $\TShyb$ & ODE, once & $\frac{\mathrm{ODEonce}}{\TShyb}$ &
  $\frac{\TSuninf}{\TShyb}$ & $\frac{\mathrm{BSA}}{\TShyb}$ \\
\midrule
0.50 & 0.179 & 0.772 & 0.268 & 1.05 & $1.99\pm0.02$ & $2.80\pm0.05$ & $1.41$ & $1.54\,[1.53,1.56]$ & $1.50\,[1.47,1.53]$ \\
0.22 & 0.417 & 0.929 & 0.221 & 1.12 & $1.66\pm0.02$ & $1.83\pm0.04$ & $1.11$ & $1.85\,[1.82,1.88]$ & $1.79\,[1.76,1.83]$ \\
0.02 & 0.631 & 1.127 & 0.176 & 1.20 & $1.36\pm0.02$ & $1.20\pm0.03$ & $0.89$ & $2.26\,[2.23,2.30]$ & $2.19\,[2.14,2.24]$ \\
\bottomrule
\end{tabular}
\end{table}

Table~\ref{tab:bias} re-fits the surrogate's population clearance on slices of the calibration study, moving the deployed model's bias and with it $\Rmech$. The hybrid’s prior table is fitted on the same slice the deployed model is calibrated on, so the study and the deployment share one population model. Regret is again monotone in $\Rmech$ and the hybrid stays below both model-free references everywhere, which neither the one solve nor the per-cycle controller does. The floor moves only over $1.08$--$1.12$ while the measured advantage moves over $1.68$--$1.85$: the margin over the floor, not the floor itself, is what the bias costs (nine slices: Table~\ref{tab:appCalibration_hyb_ode}, $33$: Figure~\ref{fig:appCalibration_hyb_bias}). The advantage also grows with the horizon, BSA/hybrid $1.33\times$ at $N = 3$ to $3.45\times$ at $N = 50$, while the floor is horizon-free (Table~\ref{tab:appCalibration_nscaling}).

\begin{table}[h]
\centering
\footnotesize
\setlength{\abovecaptionskip}{3pt}
\setlength{\belowcaptionskip}{1pt}
\renewcommand{\arraystretch}{0.92}
\setlength{\tabcolsep}{3.0pt}
\caption{Regret as the deployed model's calibration is biased ($\gamma$ = calibrated $\div$
true clearance, $11$ scored cycles, $n = 10{,}000$): three of the nine canonical slices,
floor $= \sqrt{H(\mu)/\Hmech}$,
uninformed TS $3.068$ and BSA $2.975$ constant at every slice. All nine slices, with the per-cycle controller:
Table~\ref{tab:appCalibration_hyb_ode}.}
\label{tab:bias}
\begin{tabular}{lccccccccc}
\toprule
calibration slice & $\gamma$ & $\Bmu$ & $\Rmech$ & floor &
  $\TShyb$ & ODE, once & $\frac{\mathrm{once}}{\TShyb}$ &
  $\frac{\TSuninf}{\TShyb}$ & $\frac{\mathrm{BSA}}{\TShyb}$ \\
\midrule
slowest-clearing 20\% & 0.593 & 0.304 & 0.306 & 1.08 & $1.83\pm0.02$ & $3.60\pm0.06$ & 1.97 & 1.68 & 1.63 \\
everyone (correct model) & 1.000 & 0.221 & 0.417 & 1.12 & $1.66\pm0.02$ & $1.83\pm0.04$ & 1.11 & 1.85 & 1.79 \\
fastest-clearing 20\% & 1.673 & 0.303 & 0.303 & 1.08 & $1.80\pm0.02$ & $3.37\pm0.05$ & 1.87 & 1.70 & 1.65 \\
\bottomrule
\end{tabular}
\end{table}

\setlength{\intextsep}{12pt plus 2pt minus 2pt}
\section{Related Work}
\label{sec:related}

\paragraph{Regret-information framework.}
\citet{shufaro2025bits} trade off external information against Bayesian regret: $R$ nats of prior information replace $H(\mu)$ by $H(\mu) - R$ in the rate $c\sqrt{KNH(\mu)/\log K}$, demonstrated for LLM priors. We extend this to hybrid mechanistic models: Proposition~\ref{prop:mechinf}'s measurement channel converts the model's bias $\Bmu$ into information, via $L_h$ and $\sigma_h$. Appendix~\ref{app:ext_related_work} situates this work against side-information bandits, hybrid and model-based RL, GP bandits and misspecified Bayesian bandits.

\paragraph{Adaptive 5-FU dosing.}
PK-guided adaptive 5-FU dosing~\citep{gamelin2008multivariable} and circadian chemotherapy~\citep{levi2010cancer} demonstrate clinical benefit. \citet{ernst2006} introduce tree-based fitted Q iteration and \citet{shortreed2011} apply RL to sequential treatment decisions; neither addresses continuous dose individualization. We provide the first information-theoretic sample-complexity certificate for such schemes. No policy here is evaluated against observed patient outcomes~\citep{gottesman2019guidelines} (see Appendix~\ref{app:ext_related_work} for more details).

\section{Discussion}
\label{sec:discussion}
We introduce \emph{mechanistic information} $\Rmech$ as a pre-trial measure of prior quality for mechanistic-derived priors in Bayesian adaptive control, and from it the \emph{critical bias}, a pre-trial gate on a model's ability to reduce sample complexity. Extending \citet{shufaro2025bits} to mechanistic models, we prove regret bounds tight up to $\sqrt{\log K}$ that link $\Rmech$, $\Bcrit$ and sample complexity, and quantify the burn-in-regime gain. On a calibrated in-silico 5-FU plant the hybrid prior cuts regret by $1.79\times$ against body-surface-area dosing and $1.85\times$ against an uninformed learner, stays below both under every calibration bias tested, and improves monotonically in the information the pre-dose channel supplies. %NEW
$\Bcrit$ turns ``this model captures the right biology'' into a quantitative gate.
%END NEW
Finally, priors sensitive to training-distribution shift, LLM priors among them, can lose half their mechanistic information under a shift of $\tfrac12\log K$ nats undetectable from training data alone, an argument for physically-informed priors in safety-critical adaptive control.

\paragraph{Limitations}
Theorems \ref{thm:lb}-\ref{thm:ub} cover only the finite-$K$ problem; a continuous-action extension would close the gap to the original control problem. Our 5-FU evaluation runs on an in-silico plant calibrated to published PK statistics \citep{tan2026advancing,kaldate2012,li2023}, not on patient-level outcomes, and $\Bmu$ in~\eqref{eq:bias} needs per-arm rewards no clinical dataset contains, so outside simulation it needs a reference model to supply them. The bandit treats arm means as stationary within a patient, which the reported $\sim$$20\%$ cycle-to-cycle variability, cumulative toxicity and dose delays violate, so a nonstationary formulation, a toxicity stopping rule and DPD-based arm gating are future work, and the clinical claims require a prospective trial. 

\paragraph{Declaration of AI usage.} In this work,
we used generative AI tools for designing or providing feedback on research methodology or experiments, finding related work, drafting sections of the paper, and implementing methods. 
We have not used generative AI tools to help develop theoretical models or conceptual frameworks, formulate mathematical claims, provide critical ingredients for proving mathematical claims, or propose or refine hypotheses. The rest of the required disclosure tasks are not applicable to this work. Additionally, we used generative AI tools for drafting parts of the paper, creating or editing software code, and identifying relevant literature. 
We have reviewed all AI-assisted work. We take responsibility for the final content of this work, including text, claims or artifacts produced with the aid of generative AI.

\section*{Acknowledgments}
We are grateful to Ithai Waldhorn for insightful early discussions that helped shape the core concept of this paper.

\bibliographystyle{iclr2027_conference}
\bibliography{bibliography}

%%%%%%%%%%%%%%%%%%%%%%%%%%%%%%%%%%%%%%%%%%%%%%%%%%%%%%%%%%%%
\newpage
\appendix
%% chapters/notation.tex --- reconstructed from the compiled PDF of 09-08-2026 and brought
%% up to date with the revised theory:
%%   removed: kappa_mu ("Assumption ??"), sigma_F^2, d_F, the GP residual class F, g*,
%%            eps_K, the SPRT stopping time, R_LLM, "nu = generic prior" (nu is now the
%%            measurement noise), the forward-KL shift direction;
%%   added:   Z, h, nu, sigma_nu, d_Z, rho, sigma_h^2, L_h, h_min, g_max, J_h, pi_opt,
%%            C_nu, s, B_crit(s), L_0, Lambda_t, tau_delta, mu(pi_k) as the one weight
%%            symbol, sigma_xi as the one reward-noise symbol; UH2/U corrected.
\section{Notation and conventions}
\label{app:notation}

The following table summarises the symbols used throughout the paper, in order of first
appearance. Full formal definitions are given in the corresponding sections.

\begin{longtable}{p{0.30\linewidth}p{0.62\linewidth}}
\caption{Symbols used throughout the paper, in order of first appearance.}
\label{tab:notation} \\
\toprule
\textbf{Symbol} & \textbf{Meaning} \\
\midrule
\endfirsthead
\toprule
\textbf{Symbol} & \textbf{Meaning} \\
\midrule
\endhead
\multicolumn{2}{l}{\emph{Problem primitives (\S2)}} \\
$K$ & Number of policies (bandit arms); e.g.\ $K=8$ discrete dose levels \\
$N$ & Number of interaction rounds (e.g.\ chemotherapy cycles) \\
$T$ & Horizon of the continuous-time control problem \\
$x(t) \in \R^n$ & State trajectory \\
$u(\cdot) \in \mathcal{U}$ & Continuous control input ($\mathcal{U}$ compact) \\
$\ell(x,u)$ & Instantaneous (stage) reward \\
$J(u;\mathcal{M})$ & Cumulative reward $\int_0^T \ell(x(t),u(t))\,dt$ under dynamics $\mathcal{M}$ \\
$\mathcal{M}$ & Mechanistic model (a distribution over dynamics) \\
$M^*$ & Unknown true dynamics, drawn from the prior $\rho$ \\
$\rho$ & Prior over dynamics models \\
$\hat{\mathcal{M}}$ & Calibrated (deployed) mechanistic model, $\hat{\mathcal{M}} = h^{-1}(Z)$ \\
$r_t = \bar r(\pi_t) + \xi_t$ & Observed reward at round $t$ \\
$\bar r(\pi) = J(\pi; M^*)$ & Mean reward of policy $\pi$ \\
$\xi_t$, $\sigma_\xi$ & Zero-mean reward observation noise and its standard deviation \\
$\Delta_r$ & Discretisation sub-optimality gap (Assumption 1) \\
\midrule
\multicolumn{2}{l}{\emph{The pre-trial measurement (\S2.1)}} \\
$Z = h(M^*) + \nu$ & Pre-trial measurement (e.g.\ one lead-in blood draw) \\
$h(\cdot)$ & Measurement map from models to $\R^{d_Z}$ \\
$\nu \sim \mathcal{N}(0, \sigma_\nu^2 I)$ & Measurement noise, independent of $M^*$ \\
$\sigma_\nu$ & Measurement noise standard deviation per dimension \\
$d_Z$ & Measurement dimension \\
$\sigma_h^2 = \mathrm{tr}\,\mathrm{Cov}_{M^*\sim\rho}(h(M^*))$ & Population spread of the noise-free measurement \\
$L_h$ & Identifiability constant of Assumption 3 \\
$h_{\min}$, $g_{\max}$ & Scalar-case constants in the $L_h$ appendix ($L_h = h_{\min}/g_{\max}$) \\
$J_h(\theta)$ & Measurement Jacobian ($H$ is reserved for entropy) \\
\midrule
\multicolumn{2}{l}{\emph{Policies and priors (\S2.1--2.2)}} \\
$\Pi = \{\pi_1, \ldots, \pi_K\}$ & Finite policy class \\
$\pi^{\mathrm{opt}}$ & Continuous-control optimum $\argmax_{u} J(u; M^*)$ \\
$\pistar$ & Best discrete policy (unknown; $\pistar \sim \mu$) \\
$\pihat$ & Model-recommended policy: $\argmax_{\pi\in\Pi} J(\pi;\hat{\mathcal{M}})$ \\
$\mu$ & Prior on $\Pi$ (occupancy of the optimal policy) \\
$H(\mu)$ & Shannon entropy of $\mu$ (in nats) \\
$\muhyb$ & Hybrid prior on $\Pi$ (posterior of $\mu$ given $\pihat$; Theorem 2) \\
$\mu(\pi_k)$ & Occupancy weight of arm $k$ (the single weight symbol of Eq.~3) \\
$\norm{\cdot}_{\mu}$ & Centred, occupancy-weighted reward norm of Eq.~3 \\
$\Bmu$ & Centred occupancy-weighted mechanistic bias of $\hat{\mathcal{M}}$ vs $M^*$ \\
\midrule
\multicolumn{2}{l}{\emph{Information quantities (\S2.3, \S3)}} \\
$I(X;Y)$ & Mutual information \\
$\DKL(P\|Q)$ & Kullback--Leibler divergence \\
$\mathrm{kl}(p,q)$ & Binary KL divergence \\
$\Rmech$ & Mechanistic information $I_\mu(\pistar;\pihat)$ (Definition 2) \\
$\Hmech = H(\mu) - \Rmech$ & Residual entropy \\
$C_\nu$ & Estimator-free channel bound on $\Rmech$ (Proposition 1) \\
$C(\Bmu)$ & Bias-form channel bound on $\Rmech$ (Proposition 1) \\
$s \geq 1$ & Target sample-complexity ratio in Theorem 3 \\
$\Bcrit(s)$ & Critical bias at target $s$ (Theorem 3); one-cycle point $s = N/(N-1)$ \\
\midrule
\multicolumn{2}{l}{\emph{Finite-sample / misspecification (\S4)}} \\
$\epsilon$ & Prior weight off the recommended arm in the burn-in setting \\
$\delta$ & Identification confidence threshold ($\delta$-PAC) \\
$L_0 = \log\frac{(1-\epsilon)(K-1)}{\epsilon}$ & Prior log-odds of the recommended arm \\
$\Delta$ & Sub-optimality gap in the burn-in construction \\
$\Lambda_t$ & Posterior log-likelihood ratio after $t$ rounds (Proposition 2 proof) \\
$\tau_\delta$ & $\delta$-PAC stopping time \\
$\Rtrain$, $\Rtest$ & Mechanistic information on train / test distribution \\
$P_{\mathrm{train}}, P_{\mathrm{test}}$ & Joint distributions over $(\pistar, \pihat)$ on train / test \\
$\Delta_\pi = \DKL(P_{\mathrm{train}} \| P_{\mathrm{test}})$ & Distribution shift (Proposition 3) \\
\midrule
\multicolumn{2}{l}{\emph{Performance and constants}} \\
$\BR^*(N)$ & Bayesian regret over $N$ rounds (Definition 1) \\
$\rho_{\mathrm{sc}} = H(\mu)/\Hmech$ & Asymptotic sample-complexity ratio (Eq.~7) \\
$\varepsilon$ & Mean-regret tolerance in the sample-complexity definition \\
$c, C, N_0$ & Universal constants in regret lower / upper bounds \\
$\Theta(\cdot), O(\cdot), o(\cdot)$ & Standard asymptotic notation \\
\midrule
\multicolumn{2}{l}{\emph{Clinical / experimental parameters (\S5, calibration appendix)}} \\
$n$ & Number of in-silico patients (evaluation / calibration-study cohorts) \\
$p_{\mathrm{BSA}}$ & Target attainment probability under BSA dosing ($\approx 0.20$) \\
$p_{\mathrm{opt}}$ & Target attainment probability under PK-guided dosing ($\approx 0.85$) \\
$C_{ss}$ & Steady-state plasma concentration \\
$R^2$ & Coefficient of determination (Kaldate calibration) \\
$\Delta\mathrm{AUC}, \Delta\mathrm{dose}$ & Increments in plasma AUC and dose level \\
\midrule
\multicolumn{2}{l}{\emph{Clinical acronyms}} \\
5-FU & 5-Fluorouracil (cytotoxic chemotherapy agent) \\
FOLFOX, FOLFOX6 & Folinic acid + 5-FU + Oxaliplatin combination regimens \\
AIO, FUFOX, LV5FU2 & Alternative 5-FU-containing regimens \\
BSA & Body Surface Area (standard dosing convention) \\
AUC & Area Under the Curve (plasma 5-FU concentration over $T$) \\
PK / PD & Pharmacokinetic / Pharmacodynamic \\
DPD & Dihydropyrimidine Dehydrogenase (clearance enzyme) \\
DPYD & Gene encoding DPD \\
UH2/U & Dihydrouracil to uracil ratio (DPD-phenotype marker) \\
ORR & Objective Response Rate \\
ODPM & On-Demand PK Monitoring (Capitain et al.\ adjustment protocol) \\
CV & Coefficient of Variation (intra-patient AUC CV $\approx 20\%$) \\
ACCENT & Adjuvant Colon Cancer ENd Points colorectal-trials database \\
\midrule
\multicolumn{2}{l}{\emph{Convention}} \\
$\log, \ln$ & Natural logarithm (nats) unless explicitly labelled $\log_2$ (bits) \\
\bottomrule
\end{longtable}

\paragraph{Note on the symbol $H$.} $H$ denotes Shannon entropy throughout; the
measurement Jacobian of the $L_h$ appendix is written $J_h$, and the Hamiltonian of the
optimal-control preliminaries in Appendix~D is written $\mathcal{H}$ where it appears; the
meaning is unambiguous from context.

\section{Extended Related Work}
\label{app:ext_related_work}
\paragraph{Bandits with side information.}
Prior work uses external information to reduce exploration: Bayesian transfer~\citep{lazaric2010bayesian}, meta-TS with task-distribution priors~\citep{kveton2021meta}, and TS in structured action spaces~\citep{gopalan2014thompson}. We derive the prior quantitatively from the mechanistic model's bias rather than assuming an abstract task prior.

\paragraph{Hybrid models and model-based RL.}
Neural ODEs~\citep{chen2018}, UDEs~\citep{rackauckas2020}, and PINNs~\citep{karniadakis2021,raissi2019physics} combine ODE structure with learned residuals, targeting expressive fitting rather than decision sample complexity. \citet{yin2021} address identifiability. Sample-complexity results for model-based RL, LQR~\citep{dean2020}, linear MDPs~\citep{jin2020}, tabular MDPs~\citep{azar2013minimax}, assume uniform model classes without structured physical priors. Our identifiability constant $L_h$ quantifies the inductive bias of a physical model, and $\Rmech$ gives the resulting information-theoretic gain, orthogonal to both lines of work.

\paragraph{Gaussian process bandits.} The measurement channel of Section~\ref{sec:ode_to_bandit} plays the role that the information gain $\gamma_T$ plays in GP bandits~\citep{srinivas2010gaussian,chowdhury2017kernelized}: it bounds the information available before interaction. \citet{srinivas2010gaussian} establish $\tilde{O}(\sqrt{NT\gamma_T})$ regret for GP-UCB; \citet{chowdhury2017kernelized} tighten these bounds for kernelised bandits.
\paragraph{Reinforcement learning for treatment, and its evaluation critiques.}
A parallel literature warns that RL-for-healthcare results are routinely overstated when policies are evaluated off-policy on observational data: \citet{gottesman2019guidelines} catalogue the failure modes (confounding, small effective sample sizes at the recommended actions, reward misspecification). Our results are deliberately upstream of this debate: every number in Section~\ref{sec:sim} is an on-policy simulation from a declared generative model, and no policy here is evaluated against observed patient outcomes. The certificate of Theorem~\ref{thm:phase} is a statement about the deployed model's bias, not a clinical-effectiveness claim; establishing the latter requires a prospective trial (Limitations).

\paragraph{Misspecified Bayesian bandits.} The burn-in proposition (Proposition~\ref{prop:burnin}) is related to the literature on bandits with misspecified priors. \citet{honda2010asymptotically} give an asymptotically optimal algorithm for bandits with bounded-support rewards; \citet{liu2016prior} study prior misspecification in Thompson Sampling. Our result differs in giving an explicit \emph{lower bound} on the burn-in cost in terms of the prior weight $1-\epsilon$ and the identification confidence $\delta$, derived from a martingale argument on the posterior log-odds together with the change-of-measure inequality of \citet{kaufmann2016complexity}.
\paragraph{Regret lower bounds for stochastic bandits.}
The $\Omega(\sqrt{KN})$ minimax lower bound for stochastic bandits is due to
\citet{lai1985asymptotically} and \citet{auer2002finite}; the Bayesian
$\Omega(\sqrt{KNH(\mu)/\log K})$ lower bound follows from Fano's inequality
applied to the prior~\citep{shufaro2025bits,lattimore2020}.
Theorem~\ref{thm:lb} specializes these to the hybrid setting by replacing
$H(\mu)$ with the residual $\Hmech$.
\citet{russo2018,lattimore2020} provide comprehensive treatments of
information-theoretic lower bounds in bandit settings.
\paragraph{Bayesian bandits with informative priors.}
Thompson Sampling~\citep{thompson1933} exploits the prior to reduce exploration. \citet{russo2016} and \citet{agrawal2012} establish its near-optimal regret.  \citet{russo2014posterior} characterizes the value of informative priors in the Bayesian regret framework.
Theorem~\ref{thm:ub} recovers the standard TS bound with $\Hmech$ replacing $H(\mu)$, showing that the mechanistic ODE acts as a computationally tractable
source of prior information.
\paragraph{LLMs as decision-making priors.} The LLM proposition (Proposition~\ref{prop:llm}) is motivated by recent work using LLMs as in-context planners~\citep{brohan2023can} and as priors for reinforcement learning~\citep{garg2022lisa}. \citet{shufaro2025bits} quantify the regret benefit of LLM information in question answering. Our impossibility result shows that LLM priors (among others) can lose half their mechanistic information from a KL shift of just $\frac12\log K$ nats, making them unreliable as the sole prior in safety-critical settings.
\section{Extended Discussion and Future Directions}
% \paragraph{Summary.}
% We have proved that mechanistic model information operates in two distinct regimes. In the \emph{asymptotic regime} (large $N$, $K\gg e^{\Rmech}$), the mechanistic information $\Rmech$ reduces the effective prior entropy from $H(\mu)$ to $\Hmech$, giving asymptotically matched bounds tight up to $\log^2 K$. In the \emph{burn-in regime} (small $N$, clinically relevant), the Wald-Wolfowitz SPRT bound governs, predicting $20$--$30\times$ improvement for the FOLFOX adaptive dosing setting. Both regimes are confirmed empirically.

% \paragraph{Implications for clinical trial design.}
% The two-regime structure has a direct clinical interpretation. A trial designer asks: how many cycles should the PK-monitoring window be? Our results say: at least long enough to exit the burn-in regime ($\approx5$--$10$ cycles per our simulation), but not necessarily long enough to reach the asymptotic rate (which would require $N\approx200$). The 4--6 cycle PROFUSE and ADAPTIVE-5FU window designs sit squarely in the burn-in regime, where the gain from the ODE prior is largest. The burn-in bound (Proposition~\ref{prop:burnin}) is the correct tool for window design; the asymptotic bounds (Theorems~\ref{thm:lb}--\ref{thm:ub}) are the correct tool for long-run algorithmic analysis.

\paragraph{ODE priors versus LLM priors.}
Proposition~\ref{prop:llm} provides a formal argument for preferring physically-grounded ODE priors over LLM priors in safety-critical settings. An LLM prior with $\Rtrain$ nats of mechanistic information on its training population can lose half its information from a small KL shift, which is potentially undetectable from training data alone. This distinction matters most in personalized medicine, where the test population (an individual patient) always differs from the training population (a cohort).

\paragraph{The mechanistic information quantity.}
$\Rmech = I(\pistar;\pihat)$ is a mutual information between two policy-valued variables, so a direct estimate needs samples of $\pistar$: the dose that \emph{would have been} optimal for each patient. Observed (dose, response) pairs do not contain that, since the administered dose is not the optimal one and the responses at the other doses are missing. A standard mutual-information estimator applies only once a reference model has supplied $\pistar$ per patient, which is what the in-silico tier of Appendix~\ref{sec:appCalibration} does. What is available without such a model is the ceiling $C_\nu$ of Proposition~\ref{prop:mechinf}, which needs only $\sigma_h$ and $\sigma_\nu$; this asymmetry is why the paper certifies through the ceiling rather than through a direct estimate of $\Rmech$.
Running Example~\ref{ex:5fu_mechinf} illustrates the worst-case channel bound certified
from published population constants ($\Rmech\leq0.77$~nats). Running
Example~\ref{ex:5fu_endtoend} measures the realized value on the in-silico plant
($0.42$~nats against its ceiling $C(\Bcrit) = 0.77$).
In practice this gives a pre-trial gate, with the same asymmetry: $\Bcrit(s)$ follows from
$\sigma_h$, $L_h$ and $H(\mu)$, which a population-PK study reports, so the threshold
itself is available before any trial; but the comparison of a particular model against it
needs that model's $\Bmu$, and hence the reference model discussed above. If
$\Bmu/\Bmu^\mathrm{crit}\ll1$ the ODE prior is worth using, and if $\Bmu$ exceeds the
one-cycle working point $\Bcrit(s{=}N/(N{-}1))$ the ODE cannot be expected to save even
one cycle of the trial and may not justify the complexity of the hybrid model.
What routine pre-FOLFOX calibration studies deliver directly is the channel side,
$\sigma_h$ and $\sigma_\nu$, and therefore $C_\nu$ and $\Bcrit$; the bias side takes the
additional modelling step set out in Appendix~\ref{sec:appCalibration}.

\paragraph{Additional Open directions.}
% \begin{enumerate}[label=\roman*]
% \item \textbf{Stochastic and non-smooth dynamics.}
%   Deriving $\kmu$ for systems with multiplicative noise or non-differentiable Hamiltonians via distributional robustness.
  The bound treats $\Rmech$ as known. Folding the cost of its estimation from calibration data into the regret would give an end-to-end guarantee.
% \item \textbf{Typical LLM degradation.}
%   Characterising the \emph{expected} (rather than worst-case) information loss under distributional shift for realistic LLM priors.
  When data from multiple patients are available, the effective identifiability constant $L_h$ can be improved by pooling occupancy measures. Formalizing this would extend the framework to population-level adaptive dosing.
% \end{enumerate}
\section{Additional details for Section \ref{sec:setting}}
\label{sec:app_setting}

\subsection{The identifiability constant $L_h$}
\label{sec:app_ellh}

Beyond the population quantities $\sigma_h$, $\sigma_\nu$ and $d_Z$, the only property of the mechanistic model that enters Proposition~\ref{prop:mechinf} and the rule-out half of Theorem~\ref{thm:phase} is the constant $L_h$ of Assumption~\ref{ass:h_llb},
\begin{equation}
  L_h\,\norm{{\mathcal M}-{\mathcal M}'}\;\le\;\norm{h({\mathcal M})-h({\mathcal M}')}
  \qquad\text{for all }{\mathcal M},{\mathcal M}'\in\operatorname{supp}\rho ,
  \label{eq:ellh}
\end{equation}
where $\norm{\cdot}_\mu$ is the reward-space norm of~\eqref{eq:bias}. This appendix explains what $L_h$ measures, shows that it is computable for parametric models from the sensitivity equations already used in pharmacokinetic calibration, and gives its value for the 5-FU model of Section~\ref{sec:application}.  It also records the case in which the constant is one, which is the form used in the running examples.

\paragraph{What $L_h$ measures.}
The left-hand side of~\eqref{eq:ellh} is a distance between the decisions two models lead to: it vanishes when the two reward vectors agree up to a common offset, so that $\pihat$ is the same under both, and it is large when the models disagree about which dose is best.  The right-hand side is the distance between the two measurements. Inequality~\eqref{eq:ellh} therefore states that patients whose optimal decisions differ produce measurements that differ proportionally: the pre-dose measurement resolves whatever the dose decision depends on. It is an identifiability condition on the decision-relevant part of the model.  When $L_h\to0$, the reward vector can change without any change in $Z$, the measurement carries no information about the decision, and the ceiling of Proposition~\ref{prop:mechinf} is vacuous, however precise the assay.

Only the lower constant is needed.  The proof of Proposition~\ref{prop:mechinf} converts the reward-space precision $\Bmu$ into a lower bound on the measurement noise, $\E\norm{h(\hat{\mathcal M})-h({\mathcal M}^*)}^2\ge L_h^2\Bmu^2$, and a lower bound on the noise is what a capacity ceiling requires.  An upper Lipschitz constant of $h$ would bound the noise from above and could not enter an upper bound on information.

\paragraph{Invertibility is a consequence.}
Assumption~\ref{ass:h_llb} uses the model recovered from a noiseless measurement, $\hat{\mathcal M}=h^{-1}(Z)$.  This requires no separate hypothesis: if $h({\mathcal M})=h({\mathcal M}')$ then by~\eqref{eq:ellh} $\norm{{\mathcal M}-{\mathcal M}'}_\mu=0$, so the two models have the same centered reward vector and the same recommendation. Thus $h$ is injective on decision-relevant equivalence classes of models, which is the only sense of invertibility the recommendation needs.

\paragraph{Computation for parametric models.}
Let the model be parametrized by $\theta\in\Theta\subseteq\R^p$, with reward vector $r(\theta)=(J(\pi_1;\theta),\ldots,J(\pi_K;\theta))$ and measurement map $h:\Theta\to\R^{d_Z}$.  Write $\bar r(\theta)$ for the centered, $\sqrt{\mu}$-weighted reward vector, so that $\norm{\theta-\theta'}_\mu=\norm{\bar r(\theta)-\bar r(\theta')}_2$. Integrating the forward sensitivity equations $\partial_t(\partial x/\partial\theta)=f_x\,(\partial x/\partial\theta)+f_\theta$ alongside the model under each grid policy, as is routine in pharmacokinetic software, yields the Jacobians
\[
  G(\theta)=\partial_\theta\bar r(\theta)\in\R^{K\times p},
  \qquad
  J_h(\theta)=\partial_\theta h(\theta)\in\R^{d_Z\times p}.
\]
($H$ denotes Shannon entropy throughout; the measurement Jacobian is written $J_h$.)
To first order $\norm{\theta-\theta'}_\mu\approx\norm{G(\theta)(\theta-\theta')}$ and $|h(\theta)-h(\theta')|\approx|J_h(\theta)(\theta-\theta')|$, so the linearized constant is
\begin{equation}
  L_h=\inf_{\theta\in\operatorname{supp}\rho}\;\inf_{v\neq0}\;
  \frac{|J_h(\theta)\,v|}{\norm{G(\theta)\,v}} .
  \label{eq:ellh_linearised}
\end{equation}
The infimum is positive only if $\ker J_h(\theta)\subseteq\ker G(\theta)$
for every $\theta$: every parameter direction that changes the reward
vector must also change the measurement, which forces
$\operatorname{rank}G(\theta)\le d_Z$.  This is the precise sense in which
the dimension of the pre-dose measurement limits what a mechanistic
model can transfer, and it is why $d_Z$ appears in
Proposition~\ref{prop:mechinf}.  When more parameters affect the
decision than the measurement resolves, one restricts $\theta$ to the
identified directions and treats the remaining inter-patient variation
as a $Z$-independent error $e$, which enters the margin condition and
not the ceiling.

For a single patient-specific parameter the constant is exact rather
than linearised.

\begin{lemma}[Identifiability constant, scalar parameter]
\label{lem:ellh_scalar}
Let $p=1$ and $\Theta$ an interval.  Suppose $h$ and $r$ are $C^1$ on
$\Theta$ with
\[
  |h'(\theta)|\;\ge\;h_{\min}>0
  \qquad\text{and}\qquad
  \norm{\partial_\theta\bar r(\theta)}_2\;\le\;g_{\max}
  \qquad\text{for all }\theta\in\Theta .
\]
Then Assumption~\ref{ass:h_llb} holds with $L_h=h_{\min}/g_{\max}$.
\end{lemma}
\begin{proof}
For $\theta,\theta'\in\Theta$ the mean value theorem gives
$|h(\theta)-h(\theta')|\ge h_{\min}|\theta-\theta'|$.  Writing
$\bar r(\theta)-\bar r(\theta')=\int_0^1\partial_\theta\bar r\big(\theta'+s(\theta-\theta')\big)\,ds\;(\theta-\theta')$
and applying the triangle inequality under the integral,
$\norm{\theta-\theta'}_\mu\le g_{\max}|\theta-\theta'|$.  Hence
$L_h\norm{\theta-\theta'}_\mu\le h_{\min}|\theta-\theta'|\le|h(\theta)-h(\theta')|$.
\end{proof}

\paragraph{Instantiation for 5-FU.}
The model of Section~\ref{sec:application} has a single patient-specific
parameter, the log-clearance $\theta=\log\mathrm{CL}$, and the pre-dose
measurement is a clearance estimate in log units, $Z=\theta^*+\nu$, so
$h$ is the identity and $h_{\min}=1$.  For a linear pharmacokinetic
model the cycle exposure at dose $D_k$ is $\mathrm{AUC}_k=D_ke^{-\theta}$,
and the mean reward of $\pi_k$ is the probability that this exposure,
multiplied by the intra-patient cycle-to-cycle factor, falls in the
therapeutic window; each $J(\pi_k;\theta)$ is a difference of two
cumulative distribution functions of the log cycle-to-cycle factor,
evaluated at $\log 30-\log D_k+\theta$ and $\log 20-\log D_k+\theta$.
Its derivative in $\theta$ is a difference of two densities, so
$g_{\max}$ is bounded by $\sqrt{2}$ times the peak density of the log
cycle-to-cycle factor, and is computed exactly by integrating the model
over the dose grid for $\theta$ in the support of the population prior.
By Lemma~\ref{lem:ellh_scalar}, $L_h=1/g_{\max}$ in reward units;
in parameter units, which the running examples use, the constant is one
and the certificate compares the precision of a single clearance
estimate with the inter-patient spread of log-clearance.
\subsection{Proof of Proposition \ref{prop:mechinf}}
\begin{proof}
Since $\pihat$ is a function of $Z$ and $\pistar$ is a function of $M^*$, the data-processing inequality gives $\Rmech = I(\pistar;\pihat) \leq I(M^*; Z)$. The channel $Z = h(M^*) + \nu$ has additive Gaussian noise $\nu \sim {\cal N}(0, \sigma_\nu^2 I)$, so for any input distribution with $\E\|h(M^*)\|^2 = \sigma_h^2$,
\[
I(M^*; Z) \;\leq\; \frac{d_Z}{2} \log\left(1+\frac{\sigma_{h}^2}{\,\E\|\nu\|^2 / d_Z}\right)
 \;=\; \frac{d_Z}{2} \log\left(1+\frac{\sigma_{h}^2}{\sigma_\nu^2}\right)
\]
(a Gaussian input maximizes mutual information under a power constraint. The $d_Z$ prefactor uses concavity of the logarithm across coordinates). Substituting $\hat{\cal{M}} = h^{-1}(Z)$ we have $h(\hat{\cal{M}}) - h(M^*) = \nu$, and Assumption~\ref{ass:h_llb} gives $\frac{\E\|\nu\|^2}{d_Z} \geq L_h^2 \Bmu^2$; substituting this lower bound on the noise power yields~\eqref{eq:mechinf_ub}, and since it can only enlarge the ratio, $I(M^*;Z) \leq C(\Bmu)$. We finish the proof by noting that $\Rmech \leq H(\mu)$.
\end{proof}

\begin{remark}
We note that along the proof we have used a relaxation of $\E[\norm{\nu}^2] \geq L_h^2\Bmu^2$. One might consider directly substituting $\sigma_\nu$, avoiding further relaxation. However, $\sigma_\nu$ is a measure of the spread of an auxiliary noise and is not a quantity that is actively observed or can be derived from observables. The model bias, $\Bmu$, on the other hand, can be measured directly from a calibration study, and $L_h$ is a model-dependent constant which can be calculated using $h$ alone. 
\end{remark}

\begin{remark}
    We denote $C_\nu = \frac{d_Z}{2} \log\left(1 + \frac{\sigma_h^2}{\E\norm{\nu}^2 / d_Z}\right)$. This quantity will be used later on in the rest of the Appendix.
\end{remark}

\section{Proofs for Section \ref{sec:main}}
\label{app:main_results}
\subsection{Proof of Theorem \ref{thm:lb}}
We begin by introducing Proposition 4.1 of \citet{shufaro2025bits}.
\begin{proposition}[Proposition 4.1 \citep{shufaro2025bits}]
    \label{prop:shufaro_lb}
    For any agent, the regret $K$-MAB problem is lower bounded by $c\sqrt{KH(\pi^*)T/\log K}$, under the assumption $\DKL(p_t\|P^{\pi^*}) > c'>0$, where $p_t$ is the distribution of the agent's action at time $t$ and $P^{\pi^*}$ is the action law that plays $\pistar$ with probability one. The condition holds while the algorithm has not yet concentrated on $\pistar$, which is the regime in which the bound is informative; we assume it throughout.
\end{proposition}
\begin{proof}
We begin by stating that the mechanistic recommendation $\pihat$ is computed before round~1 and revealed to the algorithm as a free side-channel. By the mutual information chain rule,
\[
  H(\pistar\mid\pihat)
  = H(\mu) - I_\mu(\pistar;\pihat)
  = H(\mu) - \Rmech
  = \Hmech.
\]
Thus, the algorithm's \emph{posterior} prior on $\pistar$ after observing
$\pihat$ has \emph{expected} entropy $\Hmech$. The lower bound follows directly by applying Proposition~\ref{prop:shufaro_lb}. The vacuous case $\Hmech\leq 0$ means the ODE uniquely identifies $\pistar$, which results in a policy with $0$ regret.
\end{proof}
\subsection{Proof of Theorem \ref{thm:ub}}
\begin{proof}
By Propositions 1 and 3 of \citet{russo2016}, Thompson Sampling with any prior $\nu$ over $K$ arms achieves $\BR^*(N;\mathrm{TS})\leq C\sqrt{KN\,H(\nu)}$. Substituting $\nu=\mu_{\mathrm{hyb}}$ and $H(\nu)=H(\mu_\mathrm{hyb})$ yields
\[
  \BR^*(N;\mathrm{TS}_\mathrm{hyb}) \;\leq\; C\sqrt{KN\,H(\mu_\mathrm{hyb})}.
\]
The prior $\mu_\mathrm{hyb}$ is the posterior after conditioning on $\pihat$ and is random.
Taking expectations over $\pihat$ and using concavity of the square root,
$\E[\sqrt{H(\mu_\mathrm{hyb})}] \leq \sqrt{\E[H(\mu_\mathrm{hyb})]} = \sqrt{\Hmech}$, since
$\Hmech = H(\pistar \mid \pihat)$ is the expected conditional entropy. Substituting
gives~\eqref{eq:ub}.
\end{proof}
\subsection{Proof of Theorem \ref{thm:phase}}
\begin{proof}
Fix the convention $s \geq 1$: the target sample-complexity ratio, $\Hmech = H(\mu)/s$. A model that reduces sample complexity by the factor $s$ must satisfy $\Hmech \leq H(\mu)/s$, i.e.\ $\Rmech \geq \left(1-\frac{1}{s}\right)H(\mu)$. By the upper bound of Proposition \ref{prop:mechinf}, this is impossible whenever
\[
\frac{d_Z}{2} \log\left(1+\frac{\sigma_{h}^2}{L_h^2B_{\mu}^2}\right) \leq \left(1-\frac{1}{s}\right)H(\mu),
\]
i.e.\ whenever $C(\Bmu) \leq (1-1/s)H(\mu)$. Solving for the bias at which equality holds:
\begin{align*}
    \log\left(1 + \frac{\sigma_{h}^2}{L_h^2B_{\mu}^2}\right) & \leq \frac{2}{d_Z} \left(1-\frac{1}{s}\right) H(\mu) \\
    \frac{\sigma_{h}^2}{L_h^2B_{\mu}^2} & \leq \exp\left(\frac{2}{d_Z} \left(1-\frac{1}{s}\right)H(\mu)\right) - 1 \\
    B_\mu & \geq 
    \frac{\sigma_h / L_h }{\sqrt{\exp\left(\frac{2}{d_Z} (1- 1 / s) H(\mu)\right) - 1}} 
    := B_{\mu}^{\mathrm{crit}}(s).
\end{align*}
Thus $\Bmu \geq \Bcrit(s)$ implies $\Rmech \leq (1-1/s)H(\mu)$ and $\Hmech \geq H(\mu)/s$; substituting into Theorem~\ref{thm:lb} gives $\BR^*(N) \geq c\sqrt{KN\Hmech/\log K} \geq c\sqrt{KNH(\mu)/(s\log K)}$, which concludes the proof.
\end{proof}
% ======================================================================
%  Appendix F -- Proofs for Section 4 (revised)
%
%  Changes relative to the current draft (minimal):
%   F.1  * Ville's inequality added (one paragraph) so that the
%          good-event step in (i) and in the prior term is rigorous;
%          both bounds pick up the factor (1-\delta) and lose one nat.
%        * Final inequality kept in the \delta^2 form; the step to
%          \log[(1-\epsilon)(K-1)/(\delta\epsilon)] was only valid for
%          \delta <= 0.17 and is removed.
%        * \tau_d -> \tau_\delta; N_hat, N_* -> N_{\pihat}, N_{\pistar}.
%   F.2  * KL direction flipped to D(P_train || P_test), which is the
%          direction under which the identity used in the construction
%          holds; part (i) (retention) removed since Prop. 3 no longer
%          states it; |S| = floor(K/2); closing sentence removed.
%
%  Corresponding edits needed in the MAIN TEXT (not included here):
%   * Prop. 2 statement: prior mass \epsilon spread uniformly over the
%     K-1 non-recommended arms; drop "assume \epsilon <= \delta"; add
%     \epsilon < (K-1)/K; algorithm class = posterior-based,
%     posterior-monotone, \delta-PAC; eq. (9) becomes
%       [ (1-\delta) log((1-\epsilon)(K-1)/\epsilon) + 2 log(1/(2.4\delta)) - 2 ]
%         / [ 4 log((1+\Delta)/(1-\Delta)) ]
%     (no \Delta in the numerator).
%   * Prop. 3 statement: \Delta_\pi = D_KL(P_train || P_test).
%   * Running Example 5 must be recomputed with the new (9).
% ======================================================================

\section{Proofs for Section \ref{sec:miss_prior}}
\label{sec:app_finite_sample}

\subsection{Proof of Proposition \ref{prop:burnin}}
\begin{proof}
We define by ${\cal{F}}_t = \sigma\left(\{r_s, a_s\}_{s=1}^t\right)$ a filtration which contains the entire play history up to time $t$.
Let $\Lambda_t$ denote the log-likelihood ratio of $\nu_{\pihat}$ against $\nu_{\pistar}$ accumulated over the first $t$ rounds,
\[
  \Lambda_t=\sum_{s\le t}\log\frac{p_{\nu_{\pihat}}(r_s\mid a_s)}{p_{\nu_{\pistar}}(r_s\mid a_s)} .
\]
The two instances differ only on arms $\pihat$ and $\pistar$, so only plays of those arms change $\Lambda_t$. The action probabilities are common to all instances and cancel in the posterior, as do the prior masses on the remaining instances, so the posterior log-odds between the two are exactly
\begin{equation}
  \log\frac{\Pr(\nu_{\pihat}\mid\mathcal F_t)}{\Pr(\nu_{\pistar}\mid\mathcal F_t)}
  = L_0+\Lambda_t .
  \label{eq:logodds}
\end{equation}
Under $\nu_{\pistar}$, a play of $\pihat$ yields a reward with mean $\tfrac12-\tfrac\Delta2$ and decreases $\Lambda$ by $d(\Delta)=\kl\big(\tfrac12-\tfrac\Delta2,\tfrac12+\tfrac\Delta2\big)=\Delta\log\frac{1+\Delta}{1-\Delta}$ in expectation. A play of $\pistar$ yields mean $\tfrac12+\tfrac\Delta2$ and, by symmetry, also decreases $\Lambda$ by $d(\Delta)$ in expectation. Writing $T_t:=N_{\pihat}(t)+N_{\pistar}(t)$ for the number of contested plays, $M_t:=\Lambda_t+d(\Delta)\,T_t$ is therefore a martingale with increments bounded by $\log\frac{1+\Delta}{1-\Delta}$. Optional stopping applies at any stopping time whose contested-play count has finite expectation; when that expectation is infinite, the corresponding bound holds trivially.

We will also need a bound on upward excursions of $\Lambda$. Under $\nu_{\pistar}$, $e^{\Lambda_t}$ is the likelihood ratio of $\nu_{\pihat}$ against $\nu_{\pistar}$ on $\mathcal F_t$, hence a nonnegative martingale with mean one, and Ville's inequality gives $\Pr(\sup_t\Lambda_t\ge x)\le e^{-x}$ for all $x\ge0$. Consequently
\begin{equation}
  \E\Big[\sup_t\Lambda_t^{+}\Big]\;\le\;\int_0^\infty e^{-x}\,dx\;=\;1 .
  \label{eq:ville}
\end{equation}

Let $\tau':=\inf\{t:L_0+\Lambda_t\le\log\frac{\delta}{1-\delta}\}$ be the first round at which the posterior odds of $\pihat$ against $\pistar$ fall to $\delta/(1-\delta)$. On the event that the algorithm outputs $\pistar$, which has probability at least $1-\delta$, its posterior on $\nu_{\pistar}$ at $\tau_\delta$ is at least $1-\delta$, so by~\eqref{eq:logodds} $\tau'\le\tau_\delta$ and $\Lambda_{\tau'\wedge\tau_\delta}=\Lambda_{\tau'}\le-L_0-\log\frac{1-\delta}{\delta}$. On the complementary event, $\Lambda_{\tau'\wedge\tau_\delta}\le\sup_t\Lambda_t^{+}$. Optional stopping at $\tau'\wedge\tau_\delta$ together with~\eqref{eq:ville} gives
\[
  d(\Delta)\,\E[T_{\tau_\delta}]\;\ge\;d(\Delta)\,\E[T_{\tau'\wedge\tau_\delta}]
  \;=\;-\E[\Lambda_{\tau'\wedge\tau_\delta}]
  \;\ge\;(1-\delta)\Big(L_0+\log\tfrac{1-\delta}{\delta}\Big)-1 .
\]

\emph{(i), prior term.} Let $\tau:=\inf\{t:L_0+\Lambda_t<0\}$ be the first round at which the posterior prefers $\pistar$ to $\pihat$. The posterior cannot reach $1-\delta$ on $\pistar$ without first preferring it, so $\tau\le\tau_\delta$ on the good event, where $\Lambda_\tau<-L_0$; on the complementary event $\Lambda_{\tau\wedge\tau_\delta}\le\sup_t\Lambda_t^{+}$. Optional stopping at $\tau\wedge\tau_\delta$ and~\eqref{eq:ville} give $d(\Delta)\,\E[T_{\tau\wedge\tau_\delta}]\ge(1-\delta)L_0-1$. For $t<\tau$ the posterior odds~\eqref{eq:logodds} are non-negative, so by posterior-monotonicity the algorithm is at least as likely to play $\pihat$ as $\pistar$ in round $t+1$; since $\{t<\tau\wedge\tau_\delta\}$ is $\mathcal F_t$-measurable, summing these conditional probabilities yields $\E[N_{\pihat}(\tau\wedge\tau_\delta)]\ge\E[N_{\pistar}(\tau\wedge\tau_\delta)]$. Hence $2\,d(\Delta)\,\E[N_{\pihat}(\tau\wedge\tau_\delta)]\ge(1-\delta)L_0-1$, and $N_{\pihat}(\tau_\delta)\ge N_{\pihat}(\tau\wedge\tau_\delta)$ pathwise.

\emph{(ii), certification term.} This is the change-of-measure argument
of ~\citep{kaufmann2016complexity}. Consider the alternative Bernoulli bandit that raises the mean of $\pihat$ alone to $\tfrac12+\tfrac\Delta2+\eta$, making $\pihat$ optimal while leaving every other arm unchanged. A $\delta$-PAC algorithm outputs $\pistar$ with probability at least $1-\delta$ under $\nu_{\pistar}$ and at most $\delta$ under the alternative, so the data-processing inequality for the stopped trajectories yields $\E[N_{\pihat}(\tau_\delta)]\,\kl\big(\tfrac12-\tfrac\Delta2,\tfrac12+\tfrac\Delta2+\eta\big)\ge\kl(\delta,1-\delta)$. Letting $\eta\downarrow0$ and using $\kl(\delta,1-\delta)\ge\log\frac{1}{2.4\delta}$ gives the second term. Both terms bound the same quantity, so their maximum does; the single-logarithm form follows from $\max\{a,b\}\ge\tfrac12(a+b)$. The regret is then lower bounded by
\[
  \Delta\,\E[N_{\pihat}(\tau_\delta)]
  \;\ge\;\frac{\Delta}{d(\Delta)}\cdot\frac12\Big(\log\frac{1}{2.4\delta}
    +\frac{(1-\delta)L_0-1}{2}\Big)
  =\frac{(1-\delta)\log\frac{(1-\epsilon)(K-1)}{\epsilon}
    +2\log\frac{1}{2.4\delta}-1}{4\log\frac{1+\Delta}{1-\Delta}} ,
\]
which is~\eqref{eq:burnin}.
\end{proof}

\subsection{Proof of Proposition \ref{prop:llm}}
\begin{proof}
Let $P,Q$ be the joint distributions of $(\pistar,\pihat)$ under the training and test distributions, respectively, with common marginal $P_{\pistar}=Q_{\pistar}=\mu$ uniform on $\Pi$, and let $\Delta_\pi=\DKL(P\|Q)$. Construct the test distribution as follows. Let $S\subset\Pi$ with $|S|=\lfloor K/2\rfloor$, and set $Q_{\pihat|\pistar}(\cdot)=P_{\pihat|\pistar}(\cdot)$ for $\pistar\in S$ and $Q_{\pihat|\pistar}(\cdot)=\mathrm{Uniform}(\Pi)$ for $\pistar\notin S$. Since the marginals of $\pistar$ agree, the chain rule gives $\DKL(P\|Q)=\E_\mu[\DKL(P_{\pihat|\pistar}\|Q_{\pihat|\pistar})]$. For $\pistar\in S$ the term is zero; for $\pistar\notin S$ it is
$\DKL(P_{\pihat|\pistar=\cdot}\|\mathrm{Uniform})=\log K-H_P(\pihat\mid\pistar=\cdot)$, the entropy of the point-conditional. Averaging over $\pistar\notin S$ under uniform $\mu$, and writing $H_P(\pihat\mid\pistar)=\E_{\pistar}[H(P_{\pihat|\pistar=\cdot})]$ for the average conditional entropy,
\[
  \DKL(P\|Q)\;=\;\frac{|\Pi\setminus S|}{K}\Big(\log K-H_P(\pihat\mid\pistar)\Big)
  \;\le\;\tfrac12\log K ,
\]
where for odd $K$ the factor $|\Pi\setminus S|/K=\lceil K/2\rceil/K$ replaces $\tfrac12$ throughout and does not affect the conclusion. Under $Q$, since for $\pistar\in S$ the conditional equals $P_{\pihat|\pistar}$ and for $\pistar\notin S$ it equals $\mathrm{Uniform}(\Pi)$,
\[
  H_Q(\pihat\mid\pistar)=\tfrac12H_P(\pihat\mid\pistar)+\tfrac12\log K .
\]
Since $H_Q(\pihat)\le\log K$,
\begin{align*}
  I_Q(\pistar;\pihat)&=H_Q(\pihat)-H_Q(\pihat\mid\pistar)\\
  &\le\log K-\big(\tfrac12H_P(\pihat\mid\pistar)+\tfrac12\log K\big)\\
  &=\tfrac12\big(\log K-H_P(\pihat\mid\pistar)\big)\\
  &\le\tfrac12\log K .
\end{align*}
The shift is $\DKL(P\|Q)=\tfrac12(\log K-H_P(\pihat\mid\pistar)) =\tfrac12\Rtrain+\tfrac12(\log K-H_P(\pihat))\le\tfrac12\log K$, since $H_P(\pihat)\ge0$. When $P_{\pihat}$ is uniform, $H_P(\pihat)=\log K$, so $\DKL(P\|Q)=\tfrac12\Rtrain$ and $\Rtest\le\tfrac12(\log K-H_P(\pihat\mid\pistar))=\tfrac12\Rtrain\le\tfrac12\log K$.
\end{proof}

\section{Additional Experiment Details}
\label{sec:app_experiment}
\paragraph{Real-data motivation.}
In a prospective 37-site German cohort of 434 patients \citep{li2023}, only $20.3\%$ of BSA-dosed patients reached the therapeutic target AUC range of $20$--$30$~mg$\cdot$h/L, while $60.6\%$ were underdosed and $19.1\%$ overdosed.
In the BSA arm of the randomized trial of \citet{gamelin2008multivariable}, only 4 of 49 assessable patients had plasma levels in the therapeutic range. The consequences are large: the \citet{gamelin2008multivariable} randomized trial reported an objective response rate of $33.7\%$ under PK-guided dosing versus $18.3\%$ under BSA dosing ($P\,{=}\,.004$) in an intent-to-treat analysis of 208 patients, with median overall survival of 22 versus 16 months, a difference that did not reach significance ($P\,{=}\,.08$), and significantly fewer toxic events 
($P\,{=}\,.003$)~\citep{gamelin2008multivariable}. A meta-analysis of four prospective trials ($N=504$) confirmed $50$--$200\%$ response rate improvement and $\approx60\%$ reduction in grade~3/4 toxicity~\citep{yang2016}.

\paragraph{The learning problem.}
The key clinical bottleneck is the number of cycles required for the algorithm to converge to a patient's optimal dose. \citet{kaldate2012} characterize the dose-AUC relationship in 187 FOLFOX6 patients (307 cycle-pair observations) as $\Delta\text{AUC} = 0.021\times\Delta\text{dose}$~mg/m$^2$ ($R^2=0.51$, $n=307$)~\citep{kaldate2012}. The moderate $R^2$ reflects the large individual variability in 5-FU clearance, approximately $100$-fold in plasma 5-FU levels under BSA-based dosing~\citep{kobuchi2020pharmacometrics}, driven by DPD genotype, sex, age, and hepatic function. The intra-individual cycle-to-cycle variability is $\approx 20\%$~\citep{li2023}, and the 10~mg$\cdot$h/L width of the target window is chosen to accommodate it~\citep{kaldate2012}.

\paragraph{Sex differential.}
Females clear 5-FU $\approx21\%$ slower (men have $26\%$ higher clearance), producing a
higher AUC under standard BSA dosing (22 vs.\ 18~mg$\cdot$h/L)~\citep{mueller2013gender}.
Clinical toxicity studies report disproportionately larger downstream effects on patient
safety, driven by the asymmetric consequences of overdosing~\citep{wagner2021}. The
algorithm individualizes rather than stratifying by sex, and sex enters only as a
covariate of the clearance model if it is included there.

\paragraph{Burn-in.}
Proposition~\ref{prop:burnin} gives a formal lower bound on the cost of a
confidently-wrong prior.
For the $K=8$ setting with $\epsilon=0.2$ and $\delta=0.01$, the expression in
\eqref{eq:burnin} evaluates to $\approx 6.02$ units of per-cycle reward (see the Running
Example on burn-in for the working computation): an identification cost, incurred before
the wrong recommendation is overturned, that is equivalent to $\approx 30$ plays of a
$\Delta$-suboptimal arm.

\paragraph{Compute.} The experiments were done on the author's personal machine using CPU only and do not require any special computing, and a full reproduction of every table and the in-silico example completes in well under an hour.

\providecommand{\leadin}[1]{\textbf{#1}}

\section{Calibration of the 5-FU example}
\label{sec:appCalibration}

This appendix instantiates the certificate twice, at two different levels of trust.
The \emph{literature tier} (Section~\ref{sec:appCalibration_literature}) evaluates
Proposition~\ref{prop:mechinf} and Theorem~\ref{thm:phase} on two constants taken directly
from a published population-PK study: the computation Running Examples~2--4 perform, and
the one a clinician could repeat today with no simulation at all. The \emph{in-silico tier}
builds the full loop of Running Example~7: the true dynamics are the two-compartment
Michaelis--Menten model of \citet{tan2026advancing}, the deployed mechanistic model is the
one-compartment linear-clearance surrogate, and every constant, including the realized
$\Rmech$ that no published table can provide, is measured on that pair and checked
against its certified ceiling.

The in-silico tier is laid out as follows.
Section~\ref{sec:appCalibration_reward} states the plant's stochastic model and the
Bernoulli reward, with the provenance of each noise magnitude.
Section~\ref{sec:appCalibration_channel} measures the pre-trial measurement channel
($\sigma_h$, $L_h$, $\Bmu$).
Section~\ref{sec:appCalibration_certificate} substitutes those into
Proposition~\ref{prop:mechinf} and Theorem~\ref{thm:phase} and checks the whole chain.
Section~\ref{sec:appCalibration_ablation} isolates the channel in a scalar end-to-end
system (Table~\ref{tab:appCalibration_ablation}).
Section~\ref{sec:appCalibration_algorithms} documents the deployed protocol and every
algorithm exactly as run.
Section~\ref{sec:appCalibration_policies} compares the five policies
(Tables~\ref{tab:rmech} and~\ref{tab:appCalibration_policytable}) and their behaviour as
the deployed model's bias grows (Table~\ref{tab:appCalibration_hyb_ode},
Figures~\ref{fig:appCalibration_distance} and~\ref{fig:appCalibration_hyb_bias}).
Section~\ref{sec:appCalibration_channelsweep} generates $\Rmech$ from the measurement
channel and scales the horizon (Tables~\ref{tab:appCalibration_channel}
and~\ref{tab:appCalibration_nscaling}, Figure~\ref{fig:appCalibration_channel}).
Section~\ref{sec:appCalibration_sensitivity} sweeps every choice made along the way
(Tables~\ref{tab:appCalibration_slice}--\ref{tab:appCalibration_stability},
Figure~\ref{fig:appCalibration_sensitivity_new}).
Section~\ref{sec:appCalibration_ofat} varies each design choice one at a time
(Table~\ref{tab:ofat}, Figures~\ref{fig:ofat_policies} and~\ref{fig:ofat_cert}).

\leadin{How to read the provenance of a number.}
Every constant in this paper falls into exactly one of five classes, and we label it as
such wherever it is introduced. \emph{Quoted}: read off a table or sentence of a cited
study. \emph{Derived}: obtained by arithmetic on quoted numbers. The steps and the
assumptions each step makes are stated, never only the result. \emph{Convention}: a
choice of units, discretisation or normalisation that we make and that no source
dictates. \emph{Identification}: an equality between two quantities that holds only under
explicitly stated conditions, and that is not a general property of the framework.
\emph{Measured in-silico}: computed on the declared plant of
Section~\ref{sec:appCalibration_channel} with fixed seeds, reproducible from the
repository. The distinction matters because only quoted and derived quantities are
available to a reader with no simulation, and only measured ones can involve $\Rmech$ or
$\Bmu$, which require knowledge of the true dynamics.

\subsection{Literature tier: the certificate from published constants}
\label{sec:appCalibration_literature}

\leadin{Population spread $\sigma_h = 0.316$ (derived from one quoted number).}
\citet{tan2026advancing} (their Table~2) report the inter-individual variance of
the eliminating parameter $V_{\max}$ as $\Omega^2 = 0.10$ on the log scale (relative
standard error $11\%$): that variance is \emph{quoted}. Turning it into $\sigma_h$ takes
two further steps, both ours. (i) The population-PK convention
$V_{\max} = \mathrm{TV}\cdot e^{\eta}$, $\eta\sim{\cal N}(0,\Omega^2)$, makes $\Omega$ a
standard deviation on the log scale, so $\sigma_h = \sqrt{0.10} = 0.316$. This is
consistent with the $31\%$ coefficient of variation the same study reports for the
covariate model. (ii) The channel of Section~\ref{sec:ode_to_bandit} carries the
\emph{decision-relevant} parameter of the deployed one-compartment surrogate, and we
identify that direction with $V_{\max}$. This is a modelling choice, not a finding of the
cited study. The same study also reports $12\%$ inter-occasion variability, which is
deliberately excluded from $\sigma_h$ because it is within-patient and is carried instead
by the cycle-to-cycle reward noise of Section~\ref{sec:application}.

\leadin{Pre-dose estimate precision $\sigma_\nu = 0.165$ (derived, four assumptions).}
This constant is \emph{not} quoted anywhere: it is the end of a four-step chain, and we
state the chain because each step is contestable. \citet{tan2026advancing}, Section~3.4,
report a mean absolute prediction error of $2.76$~mg$\cdot$h/L for the best three-sample
limited-sampling scheme under continuous infusion (their $T = 0.8, 2, 5$~h, note that
their ``MAPE'' is defined in absolute units, not as a percentage). The reference exposure
is the median AUC$_{0-\infty}$ of $21$~mg$\cdot$h/L for the $2{,}400$~mg/m$^2$
continuous-infusion regimen, from their dosing simulation (their Section~3.5 and Fig.~5,
$90\%$ prediction interval $[12, 36]$). Then
\[
  \sigma_\nu \;=\; \frac{2.76}{21}\times\sqrt{\pi/2} \;=\; 0.131\times 1.253 \;=\; 0.165 .
\]
The four assumptions, and what each costs:
(1) a mean error is normalised by a \emph{median} exposure of a single regimen, mixing
two summary statistics. Using the $2{,}800$~mg/m$^2$ panel's median of $25$ instead gives
$\sigma_\nu = 0.138$ and $C = 0.91$~nats.
(2) The resulting relative error is treated as a log-scale error, exact only for small
errors.
(3) The conversion from mean absolute error to standard deviation assumes a zero-mean
Gaussian error, $\E|X| = \sigma\sqrt{2/\pi}$. A limited-sampling estimator may be biased,
and omitting the $\sqrt{\pi/2}$ factor altogether gives $\sigma_\nu = 0.131$ and
$C = 0.96$~nats.
(4) An error in estimated AUC is identified one-to-one with an error in the eliminating
parameter, exact for linear pharmacokinetics but only approximate under the
Michaelis--Menten elimination of the plant.
Taken together these assumptions move the headline ceiling by roughly $0.1$--$0.2$~nats,
which is the honest precision of the literature tier. The certificate's qualitative verdicts are unchanged across that range. In this setting $l_\mu = \sigma_\nu$. Identification under (i)–(iii), $\hat{M} = h^{-1}(Z)$ makes the model's measurement-space error the draw's error: $C_\nu = C(l_\mu) = 0.77, \lcrit = 0.49, l_\mu = 0.165 < 0.49 $.

\subsection{The plant's stochastic model and the reward}
\label{sec:appCalibration_reward}

This subsection states the generative model of the in-silico tier in full: what is drawn,
from which distribution, on which published grounds, and how the Bernoulli reward of
Section~\ref{sec:application} arises from it. Every choice is either quoted from the
literature or labelled as an assumption.

\leadin{Patients.}
A patient is a pair $(\mathrm{BSA}, V_{\max})$. Body surface area is a clipped normal
(mean $1.90$~m$^2$, sd $0.25$, range $[1.4, 2.6]$, \citealp{tan2026advancing}, Table~1),
and the eliminating parameter follows the published covariate model the functional form of
$V_{\max} = \mathrm{TV}V_{\max}\,(\mathrm{BSA}/1.94)^{1.14}\, e^{\eta}$ with
$\eta \sim \mathcal{N}(0, \Omega^2)$ \citep[Table~S4]{tan2026advancing}: the log-normal
random effect is the standard population-PK convention and the form the source itself
fits. The plant draws $\eta \sim \mathcal{N}(0, \Omega^2)$ with $\Omega^2 = 0.14$, the
total between-patient variance the cohort inversion returns, rather than the $0.10$ quoted
for the eliminating parameter alone: the plant's log-clearance also carries the residual of
the BSA regression. The literature tier keeps the quoted value.
The true dynamics are Tan's two-compartment Michaelis--Menten ODE, integrated once
per (patient, arm) pair. The therapeutic window $[20, 30]$~mg$\cdot$h/L and the
$1{,}600$--$3{,}600$~mg/m$^2$ adjustment range are \citet{kaldate2012}'s.

\leadin{The reward is Bernoulli by construction.}
In cycle $t$ on arm $k$ the exposure the patient actually receives is
$A_t = \mathrm{AUC}_k \cdot e^{\varepsilon_t}$, $\varepsilon_t \sim
\mathcal{N}(0, \sigma_{\mathrm{cyc}}^2)$ drawn independently per cycle, and the reward is
the indicator $r_t = \mathbf{1}\{A_t \in [20, 30]\}$. Its success probability is therefore
the exact closed form
\begin{equation*}
  p_k \;=\; \Phi\!\Big(\frac{\ln(30/\mathrm{AUC}_k)}{\sigma_{\mathrm{cyc}}}\Big)
        - \Phi\!\Big(\frac{\ln(20/\mathrm{AUC}_k)}{\sigma_{\mathrm{cyc}}}\Big),
\end{equation*}
which is what the certificate constants and the regret are computed from. This is not a
second model layered on the simulator: it is the exact mean of the indicator the simulator
draws, so the reward drawn and the regret charged are consistent by construction, and
$p_k$ is never itself sampled (an estimated $p_k$ would flip the arg-max $\pistar$ for a
few percent of patients and corrupt every information quantity).
On this plant the realised attainments are $p_{\mathrm{BSA}} = 0.353$ and
$p_{\mathrm{opt}} = 0.624$ against the published $0.20$ and $0.85$ of
Section~\ref{sec:application}. The plant is kinder to standard dosing and harsher to the
ideal policy than the literature-calibrated attainments of
Section~\ref{sec:application}, and the two are not two estimates of one number.

\leadin{Three variance channels, and which of them the reward sees.}
Between-patient variation (BSA, $V_{\max}$) determines each patient's $\mathrm{AUC}_k$
vector and hence $\pistar$. Its spread across the cohort is the entropy $H(\mu)$ a learner
must resolve. Within-patient between-cycle variation ($\sigma_{\mathrm{cyc}}$) moves the
exposure the patient actually receives, so it, and only it, randomises the reward. Assay
error ($\sigma_{\mathrm{assay}}$) moves the number the laboratory reports, corrupting the
model's clearance estimate while leaving the realised exposure, and therefore the reward,
untouched. The reward integrates the biological channel alone because therapeutic benefit
depends on the exposure received, not on the reading taken from it. Both channels corrupt
the lead-in measurement, which is why $\sigma_\nu = \sqrt{\sigma_{\mathrm{cyc}}^2 +
\sigma_{\mathrm{assay}}^2} = 0.297$ enters the certificate while $\sigma_{\mathrm{cyc}}$
alone enters $p_k$.

\leadin{Why the within-patient factor is multiplicative log-normal.}
The reward's randomness is what population PK calls inter-occasion variability (IOV), and
the plant's own source models exactly that quantity: \citet{tan2026advancing} test the
first-versus-later-cycle difference ``by adding IOV on top of the IIV of $V_{\max}$'' and
estimate it at $12\%$. Their IIV enters multiplicatively as $e^{\eta}$ (the covariate model
above), so IOV enters the same way and the induced perturbation of exposure is log-normal.
the simulator's $e^{\varepsilon}$ copies the source's own convention. The
\emph{proportional} residual-error model of the 5-FU literature
($y = f\cdot(1+\epsilon)$, $25.5\%$ in \citealp{mueller2013gender}, $30$--$34\%$ in
\citealp{tan2026advancing}) describes a different quantity: error on the \emph{observed
concentration} at the sampling time, which is the concept behind
$\sigma_{\mathrm{assay}}$, not behind $\sigma_{\mathrm{cyc}}$. The choice was measured
rather than assumed: replacing the reward's link by $\mathrm{AUC}_k(1+\varepsilon)$ moves
the aggregates negligibly ($p_{\mathrm{opt}}$ $0.624 \to 0.625$, $H(\mu)$ $2.0316 \to
2.0302$ at $\sigma = 0.20$, on the evaluation cohort) while changing the arg-max arm for $10.5\%$ of patients. The
two forms converge as $\sigma$ shrinks ($97\%$ agreement at $\sigma = 0.12$), as
$e^{\varepsilon} \approx 1 + \varepsilon$ requires.

\leadin{Magnitudes, with provenance.}
$\sigma_{\mathrm{cyc}} = 0.20$ is quoted: the ``intraindividual variability of
$\sim$$20\%$'' of \citet{li2023}. It is the top of the published range, and two
independent sources corroborate the scale: \citet{tan2026advancing} fit an IOV of $12\%$
(and judge it negligible), and \citet{morawska2018} measure a $17\%$ intra-individual
coefficient of variation of the AUC/dose ratio under repeated cycles. 
The assay-side magnitude $\sigma_{\mathrm{assay}} = 0.22$ is an assumed, rounded order-of-magnitude choice, of plausible size next to the proportional residual errors the 5-FU literature reports on observed concentrations ($25.5\%$ in \citealp{mueller2013gender}, $30$–$34\%$ in \citealp{tan2026advancing}). It is a modelling choice, not an estimate of any dataset, and the split between the biological and the assay channel is an assumption, and a load-bearing one: moving variance from the assay channel to the biological channel lowers every policy's attainable reward while leaving $\sigma_\nu$
unchanged. The magnitude choice is consequential for the absolute attainment level: at
$\sigma_{\mathrm{cyc}} = 0.12$ (Tan's own IOV) the best-arm attainment would be
$p_{\mathrm{opt}} = 0.797$ against $0.624$ at the quoted $0.20$, while $H(\mu)$, and with
it the entire certificate chain, is unmoved ($2.032$ either way). The simulation keeps
the quoted, conservative $0.20$.

\subsection{In-silico tier: the measurement channel}
\label{sec:appCalibration_channel}

Every constant in this subsection is measured on the calibration-study cohort
($n_{\mathrm{study}} = 2{,}000$, seed 2024), never on the evaluation cohort, so the
certificate is a function of study data alone.

\leadin{The plant and the pre-trial measurement $Z$.}
The true dynamics are the two-compartment Michaelis--Menten model of
\citet{tan2026advancing}. The deployed mechanistic model is a one-compartment
linear-clearance surrogate. Section~\ref{sec:ode_to_bandit} posits $Z = h(M^*) + \nu$ with
$\dim Z = d_Z$. On the plant this is the lead-in blood test: cycle~1 administers the
body-surface-area dose (2{,}457~mg/m$^2$ on the grid) and draws one plasma concentration at
40~h, on the infusion plateau. Thus $h(M^*)$ is the noise-free log plateau concentration,
$\nu$ collects the biological cycle-to-cycle and assay variation ($\sigma_\nu = 0.297$ in
log units, derived), and $d_Z = 1$.

\leadin{Measurement spread $\sigma_h = 0.393$.}
$\sigma_h^2 = \mathrm{tr}\,\mathrm{Cov}(h(M^*))$ is the variance of the noise-free log
plateau concentration across the patient population, the signal the channel carries.
Measured directly on the study cohort.

\leadin{Identifiability constant $L_h = 0.765$.}
Following Appendix~\ref{sec:app_ellh}, the plant has a single patient-specific direction
(the clearance scale, perturbed through $V_{\max}$ at fixed body surface area), $h$ is the
identity on the log-concentration measurement ($h_{\min} = 1$ exactly), and
$L_h = 1/g_{\max}$ with
$g_{\max} = \sup_i \| \sqrt{\mu} \odot (\partial J/\partial z - c\mathbf{1}) \|_2 = 1.308$, the sample maximum over the 2,000 study patients,
the largest centred, occupancy-weighted sensitivity of the per-arm reward vector to the
measurement, the same norm that defines $B_\mu$ in~\eqref{eq:bias}, which is what the
proof of Proposition~\ref{prop:mechinf} requires. (Robustness: under the sup norm that
dominates it, the measured constant is $L_h = 0.547$, giving the looser
$C(B_\mu) = 1.23$ nats.) The estimate is stable to the finite-difference step:
Table~\ref{tab:appCalibration_stability}.

\leadin{Model bias $B_\mu = 0.221$, and why the definition is centred.}
Equation~\eqref{eq:bias} evaluated on the deployed model, the one-compartment surrogate
inverted once on the lead-in draw, against the plant's per-arm in-window probabilities:
centred by the occupancy-weighted mean, occupancy-weighted, and pooled across patients by
RMS. Without the centring the value is $0.250$, giving $C(\Bmu) = 0.83$~nats in place of
$0.93$. Both are reported here because an earlier version of~\eqref{eq:bias} was uncentred
and stated for a single realized pair $({\cal M}^*, \hat{\cal M})$ rather than in
expectation, so a reader comparing versions should see what the change is worth. The
centred, expectation form is primary, for three reasons. First, it is the only part of
the reward vector the results depend on: the recommendation is an $\arg\max$ and the
regret is a difference of arm means, and both are invariant to adding a constant to every
arm, so a model that is uniformly optimistic across all doses makes the same
recommendation and pays the same regret while an uncentred $\Bmu$ would charge it without
bound. Second, it is the norm in which $L_h$ is constructed
(Appendix~\ref{sec:app_ellh}), so the chain
$\E\|h(\hat{\cal M}) - h({\cal M}^*)\|^2 \geq L_h^2\Bmu^2$ used in
Proposition~\ref{prop:mechinf} holds in one norm throughout. In the uncentred norm two
models differing by a pure offset would be at positive distance while their measurements
coincide, and no positive $L_h$ could satisfy Assumption~\ref{ass:h_llb}. Third, it is
conservative: centring can only shrink $\Bmu$, hence only widen the ceiling
$C(\Bmu)$ and weaken the claim, which is the safe direction for a certificate. On this
plant the choice changes no verdict: both values sit well below the one-cycle threshold
of Theorem~\ref{thm:phase}.

\leadin{Prior entropy and mechanistic information.}
$H(\mu) = 2.034$ nats, the plug-in entropy of the measured optimal-arm occupancy (the
uniform reference is $\ln 8 = 2.079$). The deployed prior carries $\Rmech = 0.417$ nats,
measured as the plug-in mutual information between $\pistar$ and the model's recommendation, a shuffled-estimator null gives $0.012$ nats. Before the lead-in draw the recommendation
is constant across patients and $\Rmech = 0.000$ exactly: the single blood test is what
makes the ODE a prior.

\subsection{Composite certificate at the measured values}
\label{sec:appCalibration_certificate}

Substituting into the two bounds of Proposition~\ref{prop:mechinf} and the
threshold~\eqref{eq:bcrit}:
\begin{align}
  C(B_\mu) &= \tfrac{1}{2}\ln\!\left(1 + \frac{0.393^2}{0.765^2 \times 0.221^2}\right)
            = 0.929~\text{nats} \qquad \text{(bias form: the certificate)},\\[2pt]
  \Rmech &= 0.417 \;\leq\; \; C(B_\mu) = 0.929
    \qquad \text{(implied by $\Rmech \leq C_\nu$)},\\[2pt]
  \Hmech &\geq H(\mu) - C(B_\mu) = 1.10~\text{nats}
            \quad (\text{estimator-free: } 1.53, \text{ measured: } 1.62).
\end{align}

\subsection{End-to-end scalar channel ablation}
\label{sec:appCalibration_ablation}

Between the literature tier (constants only) and the plant (full ODE pair) sits the
smallest end-to-end system that exercises every part of the theory: a scalar channel plus
the dose grid, with the literature constants. A patient is a log-clearance deviation
$\theta^* \sim \mathcal{N}(0, \sigma_h^2)$. The pre-dose estimate is
$Z = \theta^* + \nu$, $\nu \sim \mathcal{N}(0, \sigma_\nu^2)$. The model with parameter
$Z$ recommends its best arm. The realized $I(\pistar;\pihat)$ is computed by Monte Carlo
over $400{,}000$ patients and Thompson sampling runs with the induced prior
$\mu_{\mathrm{hyb}} = P(\pistar \mid \pihat)$, the posterior over the best arm being approximated by reweighting a 128-sample
Monte-Carlo estimate of the best arm under per-arm Beta posteriors. Result: the realized information is
$\Rmech = 0.63$ nats against its ceiling $C(\Bmu) = C_\nu = 0.77$ ($L_h = 1$ on this
plant, and the channel bound holds, with the slack coming from the arg-max
discretization), and the induced prior
yields the regret ordering of Table~\ref{tab:appCalibration_ablation}: the induced
prior beats uninformed TS by $2.06\times$ and fixed BSA dosing by $1.86\times$. On this
\emph{linear} plant the model is exactly specified, so the whole gap is the value of the
channel. The model-misspecified regime is what the next subsections measure on the
saturable plant.

\begin{table}[h]
  \centering
  \small
  \begin{tabular}{lcc}
    \toprule
    policy & regret over $N=12$ cycles & observation it uses \\
    \midrule
    uninformed TS & $3.20 \pm 0.03$ & binary in-window reward \\
    fixed BSA dose & $2.89 \pm 0.06$ & none \\
    $\TShyb$ (induced $\mu_{\mathrm{hyb}}$) & $1.55 \pm 0.02$ & prior + binary reward \\
\bottomrule
  \end{tabular}
  \caption{Scalar channel ablation: mean $\pm$ 95\% CI over $m = 6{,}000$
  patients, seed 11. Realized $\Rmech = 0.63 \leq C_\nu = 0.77$ nats.}
  \label{tab:appCalibration_ablation}
\end{table}

\subsection{The deployed protocol and the algorithms, exactly as run}
\label{sec:appCalibration_algorithms}

This subsection records the full pipeline.

\leadin{Cohort discipline.}
Two disjoint cohorts are drawn from the population of
Section~\ref{sec:appCalibration_reward}: a \emph{calibration study} of $2{,}000$ patients
(the size of a realistic dose-ranging study, and everything the hybrid knows before
treating anyone is a function of this cohort alone) and an \emph{evaluation} cohort of $10{,}000$
patients on which every regret is scored. The study-phase construction takes the
evaluation cohort as no argument, which makes ``no leakage'' a checkable property of the
code rather than an assertion.

\leadin{Model calibration.}
The deployed one-compartment model's population clearance is fitted on the study cohort:
each study patient's apparent clearance is recovered from their noise-free plateau
concentration at the BSA dose, the BSA covariate is regressed out at the exponent
$b = 1.149$ (the plant's own $d\log \mathrm{CL}/d\log \mathrm{BSA}$ at the typical
patient, computed from Tan's constants), and the typical value is the geometric mean over
the enrolled slice. Fitted on the full cohort the calibration is correct
($\gamma = 1$). The miscalibration sweep re-fits it on the slowest- or fastest-clearing
fraction, which biases the whole clearance scale by a known $\gamma$ that multiplies
every subsequent inversion, so no consultation of the model escapes it.

\leadin{The offline study (what the hybrid knows in advance).}
The study cohort runs the identical lead-in protocol (below), giving each study patient a
model recommendation $\pihat$. Two objects are then estimated against the study patients'
known optimal arms $\pistar$, and they are all the hybrid ever receives:
(i)~the prior table $\mu_{\mathrm{hyb}}(j \mid i) = P(\pistar = j \mid \pihat = i)$,
estimated with add-one (Laplace) smoothing so no (recommendation, truth) pair sits at
probability zero, and
(ii)~the reward-shape matrix $Q[k, j] = P(r = 1 \mid \text{arm } j \text{ played},
\pistar = k)$, estimated from \emph{one binary outcome per (patient, arm)}, never from the plant's
exact probabilities (a dose-ranging study records the arm administered, and the outcomes at
the other arms are counterfactuals the simulator supplies, an idealisation stated here as
one). Rows with 20 or fewer patients fall back to the cohort mean row, and entries are clipped to
$[0.02, 0.98]$. Row $k$ of $Q$ is the dose--response curve of a patient whose optimum is
arm $k$, so a failure at arm $j$ is evidence about how far the truth is from $j$: the
structure that lets a short horizon identify one of $K$ arms. Because the table
conditions the model's recommendation on ground truth, the model contributes its
information and not its error: a systematic model error is a systematic off-diagonal
pattern that row-normalising cancels, and when the model is useless the table collapses
to the cohort marginal. This is what makes
$H(\pistar \mid \pihat) = H(\mu) - \Rmech$ hold by construction, the hypothesis of
Theorem~\ref{thm:ub}. The cost of this construction is stated as the assumption it is:
the study patients' optimal arms must be known, i.e.\ a prior dose-ranging study
established them.

\leadin{The lead-in (cycle 1, identical for every policy).}
Every patient receives the body-surface-area dose ($2{,}457$~mg/m$^2$ on the grid) and
gives one plasma sample at $40$~h, on the infusion plateau. The model is inverted on that
single reading: $\widehat{\mathrm{CL}} = \gamma R / c_{\mathrm{obs}}$ at the infusion
rate $R$, read as a plateau. Before the draw, the model's estimate is the
covariate-scaled population value $\mathrm{CL}_{\mathrm{pop}}
(\mathrm{BSA}/1.94)^{1.149}$. Because dose is per square metre while apparent clearance
scales as $\mathrm{BSA}^{1.149}$, this pre-draw model recommends the same arm to every
patient and $\Rmech = 0.000$ exactly. The lead-in cycle is not scored, its regret
($0.270$ per patient) is identical for every policy, and the model is never consulted for
its dose.

\leadin{The recommendation rule.}
Given a clearance estimate, the model predicts
$\widehat{\mathrm{AUC}}_k = d_k \cdot \mathrm{BSA} / \widehat{\mathrm{CL}}$, converts it
to a predicted in-window probability by the same closed form as
Section~\ref{sec:appCalibration_reward}, and recommends the arg-max. Nothing assumes the
reward is symmetric: a different reward would change the recommendation through the
predicted probabilities and through nothing else. One degenerate case needs a rule: a
patient the model places entirely outside the window has a probability row that
underflows to exact zeros, and those rows fall back to the arm whose predicted exposure
is closest to the window as an interval.

\leadin{The five policies over the eleven scored cycles.}
All five face the same patients and the same pre-drawn noise streams. The paired
comparisons are paired by construction. In each cycle a policy chooses an arm, regret
$p_{\pistar} - p_{\text{arm}}$ accrues in expected reward (Definition~1), the cycle's
biological factor $e^{\varepsilon_t}$ realises the exposure, the binary reward
$r_t = \mathbf{1}\{A_t \in [20,30]\}$ is observed, and the policy updates:
\begin{itemize}
  \item \emph{BSA static}: the standard-of-care arm in every cycle, reading nothing.
  \item \emph{Uninformed TS}: per-arm Beta--Bernoulli Thompson sampling from
  $\mathrm{Beta}(1 + \text{successes}, 1 + \text{failures})$, reading the binary reward.
  Posterior draws use inverse-CDF sampling on the pre-drawn uniforms so the random
  streams of the policies never desynchronise.
  \item \emph{ODE, one solve}: the mechanistic model is solved once, on the lead-in draw,
  and its recommendation $\pihat$ is played in every scored cycle. No reward is read, no
  further concentration is read and the ODE is never re-solved. It therefore holds exactly
  the information the hybrid's prior holds, the same draw, the same solve and the same
  $\pihat$, and it does not learn, which is what makes it the comparator that prices the
  reward feedback rather than the prior.
  \item \emph{Mechanistic (per-cycle controller)}: plays the model's current
  recommendation, then refits the clearance on a fresh draw of the cycle's observed
  concentration ($c_{\mathrm{obs}}$, which carries the assay error): the running estimate
  is the mean of all inverted log-clearances, so the lead-in and every scored cycle enter
  with equal weight. Twelve ODE solves per course.
  \item \emph{Hybrid (Theorem~\ref{thm:ub})}: Thompson sampling over the
  \emph{latent best arm}. Its prior is the row of the study table selected by the
  cycle-1 recommendation, $\log$-posterior
  $\log \mu_{\mathrm{hyb}}(\cdot \mid \pihat) + \sum_t \log Q[\cdot, a_t]^{r_t}
  (1 - Q[\cdot, a_t])^{1 - r_t}$, and each cycle the arm is \emph{sampled} from that
  posterior (probability matching), never arg-maxed. The model is solved once, at the
  lead-in, the prior is frozen thereafter and the concentration reading is never consulted
  again. The per-round observation is the binary reward alone, the paper's formal
  observation model.
\end{itemize}
One idealisation is explicit: the reward is computed from the true realised exposure,
while the concentration the mechanistic policy refits on carries the assay error, so the
binary readout is noise-free where the continuous one is not. This favours the
reward-reading policies and is the same idealisation the paper's reward model makes.

\subsection{Policy comparison on the plant}
\label{sec:appCalibration_policies}

Table~\ref{tab:rmech} of the main text reports four of the five policies on the calibrated
plant of Running Example~7, and Table~\ref{tab:appCalibration_policytable} reports all
five, the per-cycle controller included. This section documents how they were produced.
The protocol is identical for all: cycle 1 is the shared BSA dose with one blood test
(unscored, lead-in cost $0.27$ per patient), then 11 scored cycles. The hybrid implements the policy of~\ref{thm:ub}, with an estimated, smoothed prior table and a Q-likelihood: Thompson sampling over the latent best
arm with prior $\mu_{\mathrm{hyb}} = P(\pistar \mid \pihat)$, frozen after the single
lead-in draw, observing the binary reward alone, the paper's formal observation model. It
reads the lead-in draw once and solves the ODE once for the whole course, and thereafter
only the binary reward. The one solve reads the same draw and performs the same single
solve and then reads nothing at all, while the per-cycle controller reads the continuous
concentration and re-solves the model in every cycle.

\begin{table}[h]
  \centering
  \small
  \setlength{\tabcolsep}{4pt}
  \begin{tabular}{lccccc}
    \toprule
    policy & regret & vs.\ hybrid & lead-in draw & per-cycle readout & solves \\
    \midrule
    uninformed TS & $3.07 \pm 0.02$ & $1.85\,[1.82,1.88]$ & not read & binary & 0 \\
    BSA static & $2.97 \pm 0.04$ & $1.79\,[1.76,1.83]$ & not read & none & 0 \\
    ODE, one solve & $1.83 \pm 0.04$ & $1.11\,[1.08,1.13]$ & read & none & 1 \\
    hybrid (Thm.~\ref{thm:ub}) & $1.66 \pm 0.02$ & $1.00$ & read & binary & 1 \\
    mechanistic, per cycle & $0.73 \pm 0.02$ & $0.44\,[0.43,0.45]$ & read & continuous & 12 \\
    \bottomrule
  \end{tabular}
  \caption{All five policies at the correctly calibrated model ($\gamma = 1$, 11 scored
  cycles, $n = 10{,}000$ evaluation patients, seeds 11/2024): cumulative regret
  (mean $\pm$ 95\% CI), the comparator/hybrid ratio of means with paired-bootstrap 95\%
  intervals, and the information each policy consumes. The main text's
  Table~\ref{tab:rmech} carries the first four rows, the per-cycle controller appearing
  only here and in Table~\ref{tab:appCalibration_hyb_ode}. The unscored lead-in cycle
  costs $0.27$ for everyone and is excluded from the regret column.}
  \label{tab:appCalibration_policytable}
\end{table}

Three readings. First, what the prior alone is worth. The one solve consumes the lead-in
draw and one ODE solve and nothing else, and it already pays $1.83$ against $2.97$ for BSA
and $3.07$ for the uninformed learner, $1.62\times$ and $1.67\times$. That is the
mechanistic recommendation priced on its own, with no learning of any kind, and it is most
of the paper's headline gain. The $\Rmech = 0$ point of the channel sweep below
(Section~\ref{sec:appCalibration_channelsweep}) is the same comparator seen from the other
side: before the lead-in draw the model recommends the same arm to every patient, so a
learner that ignores the draw carries no mechanistic information at all.

Second, what the reward feedback is worth. The hybrid pays $1.66$ against the one solve's
$1.83$, a ratio of $1.11\,[1.08,1.13]$, and that gap is the whole contribution of
Theorem~\ref{thm:ub}'s posterior update at this operating point, since the two policies
enter cycle~1 with identical information. The regret-ratio floor Theorem~\ref{thm:lb}
predicts at the measured $\Hmech = 1.62$ is $\sqrt{H(\mu)/\Hmech} = 1.12$, and the
observed $1.85$ against the uninformed learner exceeds it, as a minimax floor should be
exceeded. The hybrid holds two objects the one solve does not, the study table
$P(\pistar \mid \pihat)$ and the shape matrix $Q$, and both require a dose-ranging study
whose patients' optimal arms are known, which is a real cost and is stated as an
assumption in Section~\ref{sec:appCalibration_algorithms}. 

Third, the per-cycle controller, which is why it stays in this appendix. At the calibrated
model it pays $0.73$, $0.44\times$ the hybrid, but it reads a continuous concentration and
re-solves the ODE in all twelve cycles, an information budget the theory does not assume
and a clinic would have to fund, and under bias it is the most fragile of the five
(Table~\ref{tab:appCalibration_hyb_ode}).

\leadin{The two-compartment world versus the one-compartment model, under calibration
bias.}
The deployed surrogate differs from the world in two ways. The \emph{topology} gap,
dropping the peripheral compartment, is worth under $1\%$ of exposure (measured by
switching the plant's integrator to one compartment). The real misspecification is the
\emph{linearisation}: the surrogate's dose--exposure line is exact at the calibration dose
and drifts as DPD saturation bends the true curve, and re-fitting the surrogate's
population clearance on a slice of the cohort tilts the whole line.
Figure~\ref{fig:appCalibration_distance} shows that behavioral distance directly (the
model$-$world AUC error across the dose range, one line per calibration slice), and it
maps monotonically onto the measured centred bias, $B_\mu = 0.22 \rightarrow 0.30$.

Figure~\ref{fig:appCalibration_hyb_bias} and Table~\ref{tab:appCalibration_hyb_ode} show
what that distance costs. The hybrid’s prior table is fitted on the same slice the deployed model is calibrated on, so the dose-ranging study and the deployment describe one population model. Under that discipline the hybrid's regret grows from $1.66$ at the correct
model to $1.83$ at the harshest slice, about a tenth, because the prior enters once, so
once the rewards accumulate they take over and the initial misdirection is not
re-asserted. At every slice it remains below uninformed TS ($3.07$) and BSA ($2.97$).

Both mechanistic policies degrade further and faster, and they separate the two ways a
biased model can be paid for. The one solve, which carries the drifted recommendation into
every cycle but never re-estimates it, moves $1.83 \rightarrow 3.60$ over the sweep. It
crosses BSA between the slowest-clearing $40\%$ and $35\%$ and the fastest-clearing $35\%$
and $30\%$ slices and uninformed TS between the slowest $35\%$ and $30\%$ and the fastest
$30\%$ and $25\%$, so a model good enough to beat
standard of care at $\gamma = 1$ stops being so once $\Bmu$ passes roughly $0.28$, while
the hybrid built on the same recommendation never crosses either reference. The per-cycle
controller, which re-asserts the wrong model twelve times per course, degrades roughly
linearly in $\Bmu$ and ends $1.8\times$ worse than the hybrid at the harshest slice. Its
fragility decomposes: most of it is already in the one solve, that is, in the biased model
itself, and only the remainder is the cost of re-asserting it. The two cross between the
$50\%$ and $45\%$ slices, on both sides. At the $50\%$ slices and wider the per-cycle
controller is still the better of the two ($2.56$ against $2.61$ slow, $2.39$ against
$2.45$ fast), since its fresh readings partly correct the population term, and from the
$45\%$ slices inward it is the worse ($2.83$ against $2.74$ slow, $2.63$ against $2.57$
fast), since the readings are inverted through a model whose whole clearance scale is
wrong. Information transferred once, as a prior, degrades gracefully. A model re-asserted
every cycle re-asserts the drift, and a model played open-loop keeps the drift it started
with.

\begin{figure}[!htbp]
  \centering
  \includegraphics[width=0.72\linewidth]{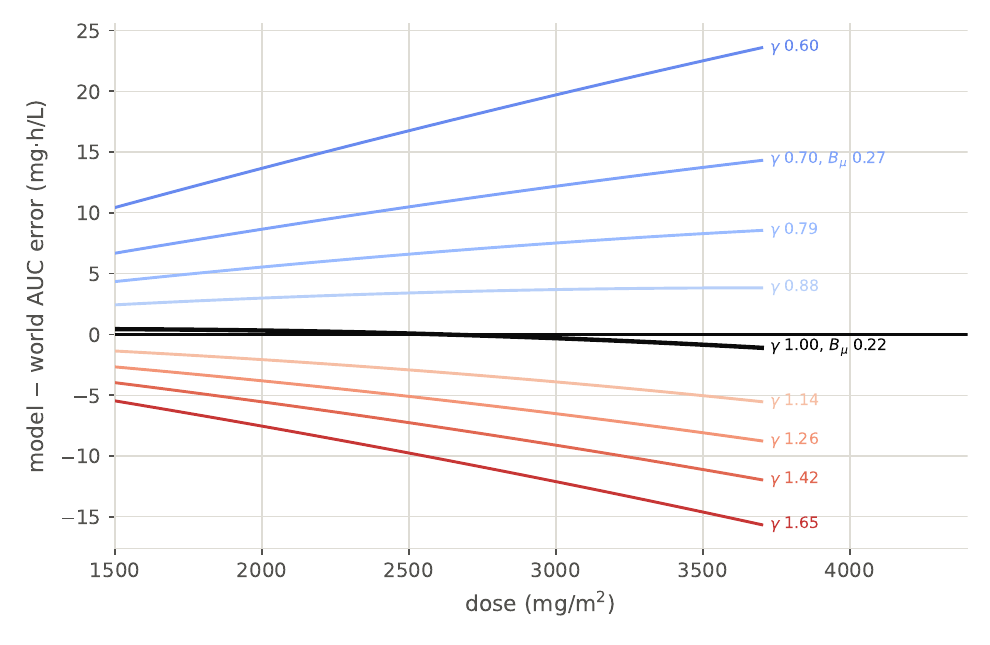}
  \caption{How far the deployed one-compartment model is from the two-compartment world:
  model$-$world AUC prediction error across the dose range for the median-clearance
  patient, one line per calibration slice, labelled with $\gamma$ (calibrated $\div$ true
  clearance) and the measured centred $B_\mu$. Under-estimated clearance
  ($\gamma < 1$, blue) reads too much exposure and pushes doses down, while
  over-estimated clearance (red) does the reverse. The correctly calibrated model (black)
  is exact at the calibration dose and drifts only with saturation.}
  \label{fig:appCalibration_distance}
\end{figure}

\begin{figure}[!htbp]
  \centering
  \includegraphics[width=0.72\linewidth]{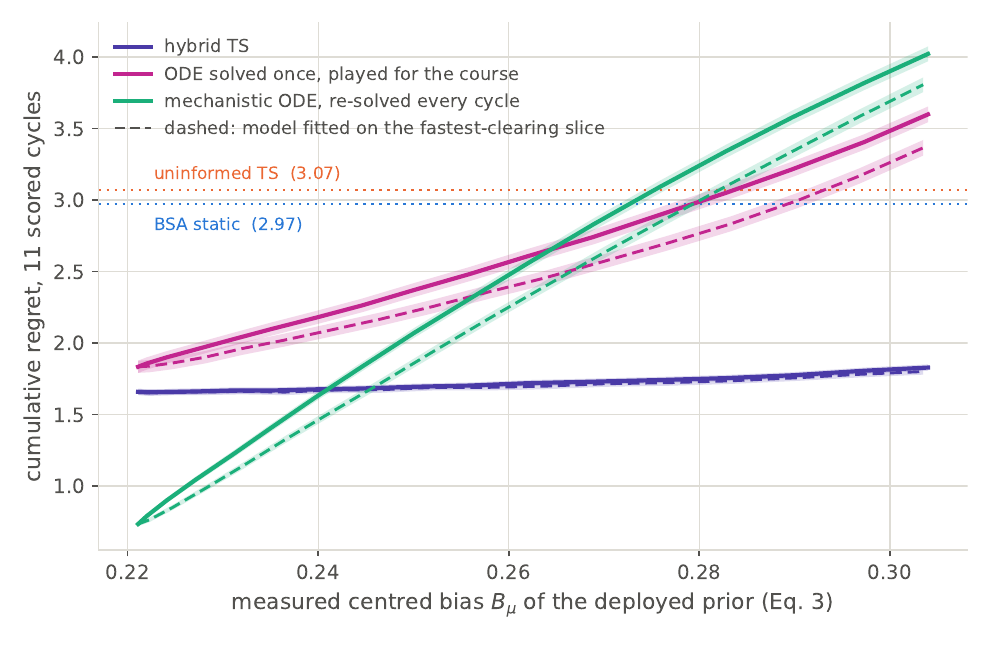}
  \caption{What the model's bias costs each deployment: cumulative regret over the 11
  scored cycles as a function of the measured centred $B_\mu$: mean curves with shaded
  95\% bands over the dense calibration sweep (33 slices). Solid: model fitted on the slowest-clearing
  slice. Dashed: on the fastest. For the hybrid the two directions nearly coincide (they
  differ by at most $0.03$ at matched slices), so $B_\mu$, not the direction of the
  miscalibration, is what prices its regret, while the per-cycle controller separates
  further, by up to $0.235$. Three model-based curves are shown. The hybrid
  (Theorem~\ref{thm:ub}: TS with the ODE prior, one draw, one solve, prior table fitted on the same slice as the deployed model, rises by about a tenth across
  the sweep and stays below both model-free references (dotted, constant). The ODE solved
  once on the same draw, playing $\pihat$ in every cycle with no feedback, rises from
  $1.83$ to $3.60$ and crosses both references near $\Bmu \approx 0.28$. The per-cycle % [caption updated: not listed in Part 2]
  controller degrades linearly, is the best of the three near $\gamma = 1$ and the worst
  of the three from the $45\%$ slices inward, where it also passes the one solve.}
  \label{fig:appCalibration_hyb_bias}
\end{figure}

\begin{table}[h]
  \centering
  \footnotesize
  \setlength{\tabcolsep}{2pt}
  \begin{tabular}{lcccccccc}
    \toprule
    calibration slice & $B_\mu$ & $\Rmech$ & floor & $\TShyb$ & ODE, once &
      $\frac{\mathrm{once}}{\TShyb}$ & per-cycle ODE &
      $\frac{\TSuninf}{\TShyb}$ \\
    \midrule
    slowest 20\% & 0.304 & 0.306 & 1.08 & {$1.83\pm0.02$} & $3.60\pm0.06$ & {$1.97\,[1.94,2.00]$ } & $4.02\pm0.05$ & {1.68} \\
    slowest 40\% & 0.276 & 0.364 & 1.10 & {$1.74\pm0.02$} & $2.90\pm0.05$ & {$1.66\,[1.63,1.69]$} & $3.09\pm0.04$ & {1.76} \\
    slowest 60\% & 0.250 & 0.396 & 1.11 & {$1.69\pm0.02$} & $2.37\pm0.05$ & {$1.40\,[1.37,1.42]$} & $2.07\pm0.03$ & {1.81} \\
    slowest 80\% & 0.231 & 0.410 & 1.12 & {$1.67\pm0.02$} & $2.03\pm0.05$ & {$1.21\,[1.19,1.24]$} & $1.22\pm0.02$ & {1.84} \\
    everyone ($\gamma=1$) & 0.221 & 0.417 & 1.12 & $1.66\pm0.02$ & $1.83\pm0.04$ & $1.11\,[1.08,1.13]$ & $0.73\pm0.02$ & 1.85 \\
    fastest 80\% & 0.232 & 0.408 & 1.12 & {$1.66\pm0.02$} & $1.96\pm0.04$ & {$1.18\,[1.16,1.20]$} & $1.14\pm0.02$ & {1.85} \\
    fastest 60\% & 0.252 & 0.387 & 1.11 & {$1.69\pm0.02$} & $2.26\pm0.05$ & {$1.34\,[1.31,1.36]$} & $1.94\pm0.03$ & {1.82} \\
    fastest 40\% & 0.277 & 0.356 & 1.10 & {$1.72\pm0.02$} & $2.70\pm0.05$ & {$1.57\,[1.54,1.60]$} & $2.88\pm0.04$ & {1.78} \\
    fastest 20\% & 0.303 & 0.303 & 1.08 & {$1.80\pm0.02$} & $3.37\pm0.05$ & {$1.87\,[1.83,1.90]$} & $3.81\pm0.05$ & {1.70} \\
\bottomrule
  \end{tabular}
  \caption{The nine canonical calibration slices in full: the hybrid
  (Theorem~\ref{thm:ub}'s algorithm, one draw, one solve) against the ODE solved once on
  the same draw, the per-cycle mechanistic controller (twelve of each) and the two
  model-free references, with the certificate quantities measured in each slice. Regret over 11 scored cycles
  (mean $\pm$ 95\% CI, $n = 10{,}000$ evaluation patients). The once/hybrid column is a
  ratio of means with paired-bootstrap 95\% intervals, and the per-cycle controller's ratio
  against the hybrid runs from $0.44\times$ at $\gamma = 1$ to $2.20\times$ at the slowest
  $20\%$ slice. The floor is Theorem~\ref{thm:lb}'s
  $\sqrt{H(\mu)/\Hmech}$ with $\Hmech = H(\mu) - \Rmech$, $H(\mu) = 2.034$~nats. The
  slices' $\gamma$ values are those of Table~\ref{tab:bias} and
  Figure~\ref{fig:appCalibration_distance}. Model-free references, constant at
  every bias: uninformed TS $3.068$ and BSA $2.975$, so BSA/hybrid is $0.970$ times the
  uninformed ratio in every row. The hybrid’s prior table is fitted on the same slice as the deployed model. Table~\ref{tab:bias} of the main text reports three of these
  rows, and Figure~\ref{fig:appCalibration_hyb_bias} shows the full 33-slice sweep.}
  \label{tab:appCalibration_hyb_ode}
\end{table}

\subsection{Channel-design scenarios and horizon scaling}
\label{sec:appCalibration_channelsweep}

\leadin{Regret against $\Rmech$ generated by the channel.}
Table~\ref{tab:appCalibration_channel} is the full version of Table~\ref{tab:channel}:
$\Rmech$ scenarios produced at the calibrated model by the pre-dose measurement channel,
the assay precision $\sigma_{\mathrm{assay}}$ over $0.02$--$0.50$ at one draw and the
number of pooled lead-in draws $d_Z \in \{1,2,3\}$ at the quoted
$\sigma_{\mathrm{assay}} = 0.22$. This is scenario construction over the channel, not a
sensitivity of a calibrated constant and not model bias: the deployed model is the
correctly calibrated one in every row, and the miscalibration axis is
Table~\ref{tab:appCalibration_hyb_ode}'s. The $\Rmech = 0$ row is the \emph{uninformed}
reference: before the lead-in draw the model's recommendation is constant across patients
and $\Rmech = 0.000$ exactly, so a learner that does not read the draw is uninformed
Thompson sampling, and the hybrid's regret in that row is that reference's $3.07$.
Figure~\ref{fig:appCalibration_channel} plots the same sweep. The certificate holds in all
eleven rows ($\Rmech \leq C(\Bmu)$, $\Bmu < \Bcrit$), with the tightest margin
$3.02\times$ over $\Bcrit = 0.810$ at one draw. Precision saturates below
$\sigma_{\mathrm{assay}} \approx 0.05$ because the residual $\sigma_{\mathrm{cyc}} = 0.20$
is patient biology: the maximum a single draw carries on this plant is $0.631$ nats.
Pooling draws does move the ceiling, to $0.706$ nats at three draws, and the literature
ceiling of $0.77$ nats (Running Example~\ref{ex:5fu_mechinf}, a three-sample scheme with
$\sigma_\nu = 0.165$) is not attained on this plant.

This sweep is also where the hybrid loses, and the table says so plainly. The one solve
falls faster in $\Rmech$ than the hybrid does, from $2.80$ at
$\sigma_{\mathrm{assay}} = 0.50$ to $1.20$ at $0.02$ against the hybrid's $1.99$ to
$1.36$, because a sharper channel improves the recommendation it plays directly while the
hybrid still spends cycles sampling around that recommendation. The two are statistically
tied at $\sigma_{\mathrm{assay}} = 0.15$ ($1.53$ each, ratio $1.00\,[0.98,1.02]$), the one
solve is ahead at $0.10$ and below ($0.94$ at $0.10$, $0.90$ at $0.05$) and ahead at two
and three pooled draws ($0.88$ and $0.77$). So the crossing lies between
$\sigma_{\mathrm{assay}} = 0.10$ and the quoted operating point $0.22$, that is between
$\Rmech = 0.56$ and $0.42$~nats, and the paper's operating point sits just on the hybrid's
side of it. The reading is not that the posterior update is worthless at high $\Rmech$ but
that probability matching keeps paying for exploration whose value shrinks as the prior
concentrates, which is the finite-$N$ face of Theorem~\ref{thm:ub}'s asymptotic statement.
Under bias (Table~\ref{tab:appCalibration_hyb_ode}) the ordering never reverses, so the
hybrid's advantage is an advantage against drift, not against
information and short horizons.

\begin{table}[!htbp]
  \centering
  \footnotesize
  \setlength{\tabcolsep}{1.8pt}
  \begin{tabular}{ccccccccccc}
    \toprule
    $\sigma_{\mathrm{assay}}$ & $d_Z$ & $\sigma_\nu$ & $\Rmech$ & $C(\Bmu)$ & $\Bmu$ &
      floor & $\TShyb$ & ODE, once & $\frac{\mathrm{once}}{\TShyb}$ &
      $\frac{\TSuninf}{\TShyb}$ \\
    \midrule
    --- & --- & --- & $0$ & --- & --- & 1.00 & $3.07\pm0.02$ & --- & --- & $1.00$ \\
    0.50 & 1 & 0.539 & 0.179 & 0.772 & 0.268 & 1.05 & $1.99\pm0.02$ & $2.80\pm0.05$ & $1.41\,[1.38,1.43]$ & $1.54\,[1.53,1.56]$ \\
    0.40 & 1 & 0.447 & 0.246 & 0.805 & 0.257 & 1.07 & $1.90\pm0.02$ & $2.53\pm0.05$ & $1.33\,[1.31,1.35]$ & $1.61\,[1.59,1.63]$ \\
    0.30 & 1 & 0.361 & 0.330 & 0.862 & 0.240 & 1.09 & $1.78\pm0.02$ & $2.17\pm0.05$ & $1.22\,[1.20,1.24]$ & $1.72\,[1.70,1.74]$ \\
    0.22 & 1 & 0.297 & 0.417 & 0.929 & 0.221 & 1.12 & $1.66\pm0.02$ & $1.83\pm0.04$ & $1.11\,[1.08,1.13]$ & $1.85\,[1.82,1.88]$ \\
    0.15 & 1 & 0.250 & 0.508 & 1.005 & 0.202 & 1.15 & $1.53\pm0.02$ & $1.53\pm0.04$ & $1.00\,[0.98,1.02]$ & $2.01\,[1.98,2.04]$ \\
    0.10 & 1 & 0.224 & 0.564 & 1.063 & 0.189 & 1.18 & $1.45\pm0.02$ & $1.35\pm0.04$ & $0.94\,[0.92,0.96]$ & $2.12\,[2.09,2.15]$ \\
    0.05 & 1 & 0.206 & 0.614 & 1.110 & 0.180 & 1.20 & $1.38\pm0.02$ & $1.24\pm0.03$ & $0.90\,[0.88,0.92]$ & $2.23\,[2.19,2.26]$ \\
    0.02 & 1 & 0.201 & 0.631 & 1.127 & 0.176 & 1.20 & $1.36\pm0.02$ & $1.20\pm0.03$ & $0.89\,[0.87,0.90]$ & $2.26\,[2.23,2.30]$ \\
    \bottomrule
  \end{tabular}
  \caption{Regret against mechanistic information generated by the measurement channel at
  the calibrated model ($\gamma = 1$, $11$ scored cycles, $n = 10{,}000$ evaluation
  patients, regret mean $\pm$ 95\% CI, ratios of means with paired-bootstrap 95\%
  intervals). Upper block: the $\sigma_{\mathrm{assay}}$ axis at one draw, with
  $\sigma_\nu = \sqrt{\sigma_{\mathrm{cyc}}^2 + \sigma_{\mathrm{assay}}^2}$ and
  $\sigma_{\mathrm{cyc}} = 0.20$. Lower block: the pooled-draw axis at the quoted
  $\sigma_{\mathrm{assay}} = 0.22$. The floor is Theorem~\ref{thm:lb}'s
  $\sqrt{H(\mu)/\Hmech}$, $H(\mu) = 2.034$~nats. The model-free references are constant
  throughout, uninformed TS $3.07$ and BSA $2.97$. First row: the pre-draw model, which
  recommends one arm to everyone, so $\Rmech = 0.000$ and the hybrid \emph{is} uninformed
  TS. $^\dagger$~per-draw ceiling $C(\Bmu)/d_Z$. The raw values $2.179$ and $3.639$ carry
  the $d_Z$ prefactor.}
  \label{tab:appCalibration_channel}
\end{table}

\begin{figure}[!htbp]
  \centering
  \includegraphics[width=0.9\linewidth]{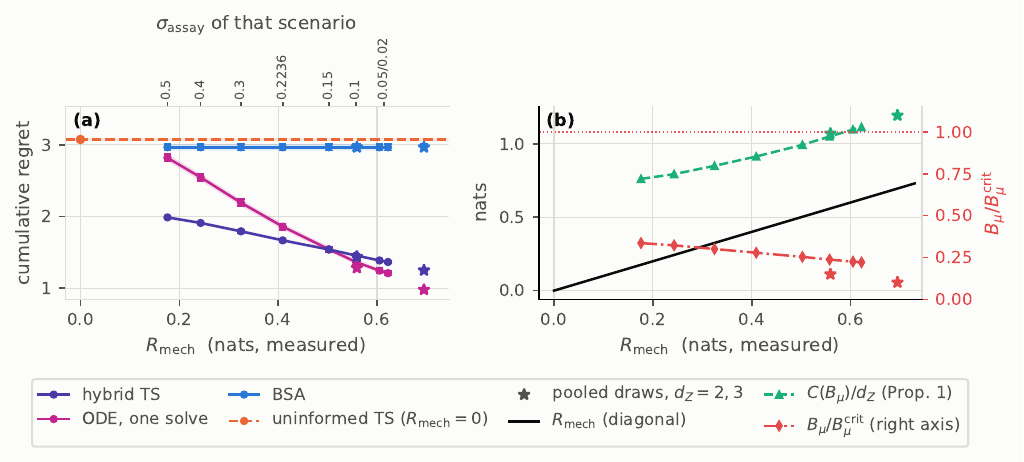}
  \caption{The channel sweep of Table~\ref{tab:appCalibration_channel}: cumulative regret
  over the 11 scored cycles against the measured $\Rmech$ that the pre-dose channel
  supplies, at the correctly calibrated model. The two model-free references do not read
  the draw and are flat, while the model-based policies fall monotonically in $\Rmech$,
  the one solve steeply enough to pass the hybrid between $\Rmech = 0.42$ and $0.56$~nats.} % [caption updated: not listed in Part 2]
  \label{fig:appCalibration_channel}
\end{figure}

\leadin{Horizon scaling at the calibrated model.}
Table~\ref{tab:appCalibration_nscaling} varies the number of scored cycles at the
correctly calibrated model with everything else held fixed. The advantage grows with the
horizon, from $1.33\times$ over BSA at $N = 3$ to $3.45\times$ at $N = 50$, while the
certificate does not move: $\Hmech = 1.617$~nats and the floor $1.12$ in every row. The
two are not in tension, because the floor is a statement about sample complexity and not
about the ratio of $N$-cycle regrets. The uninformed learner must resolve eight arms from
one bit per cycle, so it keeps accruing regret where the hybrid's prior has already
localised the arm, and BSA, which never adapts, accrues linearly. Only $N \leq 12$ is a
clinical course, and longer horizons are reported for the scaling and have no per-patient
reading.

This axis is also where the hybrid loses, and the table states it. The one solve accrues
linearly at $0.167$ per cycle, since it plays one fixed arm, so it is \emph{below} the
hybrid at the short horizons: $0.50$ against $0.61$ at $N = 3$ (ratio
$0.82\,[0.81,0.84]$) and $1.00$ against $1.08$ at $N = 6$ ($0.93$), with the cumulative
crossover at $N = 8$ and the hybrid ahead by $1.11\times$ at the reported $N = 11$,
$1.39\times$ at $20$ and $2.13\times$ at $50$. The mechanism is Theorem~\ref{thm:ub}'s
sampling step. The hybrid draws its arm from $\mu_{\mathrm{hyb}}(\cdot \mid \pihat)$ rather
than playing that row's mode, so it plays $\pihat$ itself with probability $0.360$ in
cycle~1 and pays $0.225$ in that cycle against the one solve's $0.167$. It is ahead on a
per-cycle basis only from cycle~5, by which point the posterior has concentrated on a
better arm than $\pihat$ and stays there. A clinic with three cycles to spend should
therefore play the model and not the posterior, which is the finite-$N$ caveat that
Proposition~\ref{prop:burnin} and this table together draw.

\begin{table}[!htbp]
  \centering
  \small
  \setlength{\tabcolsep}{4pt}
  \begin{tabular}{ccccccc}
    \toprule
    $N$ & $\TShyb$ & $\TSuninf$ & BSA & ODE, once &
      $\frac{\mathrm{once}}{\TShyb}$ & $\frac{\TSuninf}{\TShyb}$ \\
    \midrule
    3  & $0.61\pm0.01$ & $0.96\pm0.01$ & $0.81\pm0.01$  & $0.50\pm0.01$ & $0.82\,[0.81,0.84]$ & $1.57\,[1.54,1.60]$ \\
    6  & $1.08\pm0.01$ & $1.81\pm0.01$ & $1.62\pm0.02$  & $1.00\pm0.02$ & $0.93\,[0.91,0.94]$ & $1.67\,[1.65,1.70]$ \\
    11 & $1.66\pm0.02$ & $3.07\pm0.02$ & $2.97\pm0.04$  & $1.83\pm0.04$ & $1.11\,[1.08,1.13]$ & $1.85\,[1.82,1.88]$ \\
    20 & $2.40\pm0.03$ & $4.98\pm0.03$ & $5.41\pm0.08$  & $3.33\pm0.08$ & $1.39\,[1.36,1.42]$ & $2.08\,[2.05,2.11]$ \\
    50 & $3.92\pm0.05$ & $9.63\pm0.06$ & $13.52\pm0.20$ & $8.34\pm0.19$ & $2.13\,[2.08,2.18]$ & $2.46\,[2.42,2.49]$ \\
    \bottomrule
  \end{tabular}
  \caption{Scored-cycle scaling at the calibrated model ($\gamma = 1$): cumulative regret
  (mean $\pm$ 95\% CI over $n = 10{,}000$ evaluation patients) and ratios of means with
  paired-bootstrap 95\% intervals. The certificate is horizon-free:
  $\Hmech = 1.617$~nats and Theorem~\ref{thm:lb}'s floor $1.12$ in every row. The one
  solve is linear in $N$ at $0.167$ per cycle and is below the hybrid at $N = 3$ and
  $N = 6$, the cumulative crossover falling at $N = 8$. The per-cycle controller is
  omitted here and reported in Table~\ref{tab:appCalibration_policytable}. The biased
  counterpart of this axis is the $N$ block of Table~\ref{tab:ofat}. }
  \label{tab:appCalibration_nscaling}
\end{table}

\subsection{Sensitivity analysis}
\label{sec:appCalibration_sensitivity}

Every choice made above is swept with the remaining configuration held fixed.
Figure~\ref{fig:appCalibration_sensitivity_new} shows the certificate across working
points and under miscalibration.

\leadin{Miscalibration (the axis that matters clinically).}
The surrogate's population clearance is re-fitted on a slice of the calibration study
(the slowest- or fastest-clearing $80/60/40/20\%$) and the whole protocol re-run
(Table~\ref{tab:appCalibration_slice}). The measured bias crosses the one-cycle threshold
nowhere. The measured $\Rmech$ falls exactly as the shrinking $C(B_\mu)$ predicts, and the
chain $\Rmech \leq C_\nu \leq C(B_\mu)$ holds in every slice. Note that $C_\nu$ is
constant by construction (it depends only on the channel, not on the deployed model),
so the bias form $C(B_\mu)$ is the bound that tracks miscalibration.

\begin{table}[h]
  \centering
  \small
  \setlength{\tabcolsep}{4pt}
  \begin{tabular}{lcccccc}
    \toprule
    calibration slice & $\gamma$ & $B_\mu$ & $C(B_\mu)$ & $\Rmech$ &
    chain & $B_\mu < {\Bcrit}^{\text{1-cycle}}$ (not ruled out) \\
    \midrule
    slowest-clearing 20\% & 0.593 & 0.304 & 0.676 & 0.306 & yes & yes \\
    slowest-clearing 40\% & 0.702 & 0.276 & 0.750 & 0.364 & yes & yes \\
    slowest-clearing 60\% & 0.792 & 0.250 & 0.827 & 0.396 & yes & yes \\
    slowest-clearing 80\% & 0.879 & 0.231 & 0.893 & 0.410 & yes & yes \\
    everyone (correct model) & 1.000 & 0.221 & 0.929 & 0.417 & yes & yes \\
    fastest-clearing 80\% & 1.140 & 0.232 & 0.889 & 0.408 & yes & yes \\
    fastest-clearing 60\% & 1.266 & 0.252 & 0.821 & 0.387 & yes & yes \\
    fastest-clearing 40\% & 1.418 & 0.277 & 0.747 & 0.356 & yes & yes \\
    fastest-clearing 20\% & 1.673 & 0.303 & 0.677 & 0.303 & yes & yes \\
\bottomrule
  \end{tabular}
  \caption{The certificate under miscalibration: $\gamma$ = calibrated/true clearance,
  $B_\mu$ is the centred bias of Eq.~\eqref{eq:bias}, $C_\nu = 0.506$ throughout, and the
  threshold is ${\Bcrit}^{\text{1-cycle}} = 0.810$ ($s = 12/11$, the full 12-cycle course:
  one lead-in and eleven scored).}
  \label{tab:appCalibration_slice}
\end{table}

\leadin{Discretization $K$.}
Table~\ref{tab:appCalibration_K}: $\sigma_h$, $L_h$, $B_\mu$ and $C(B_\mu)$ are
essentially flat in $K$. The one-cycle threshold tightens slowly as the measured occupancy
entropy $H(\mu)$ grows. The classification at the working point is unchanged across
$K \in \{4,\ldots,16\}$, and the same flatness holds on a drifted model
(Table~\ref{tab:ofat}).

\begin{table}[h]
  \centering
  \small
  \begin{tabular}{cccccc}
    \toprule
    $K$ & $H(\mu)$ & $L_h$ & $B_\mu$ & $C(B_\mu)$ &
    ${\Bcrit}^{\text{1-cycle}}$ \\
    \midrule
    4 & 1.374 & 0.737 & 0.223 & 0.947 & 1.046 \\
    6 & 1.767 & 0.763 & 0.221 & 0.929 & 0.880 \\
    8 & 2.034 & 0.765 & 0.221 & 0.929 & 0.810 \\
    10 & 2.228 & 0.769 & 0.221 & 0.925 & 0.764 \\
    12 & 2.385 & 0.767 & 0.221 & 0.923 & 0.732 \\
    16 & 2.624 & 0.767 & 0.221 & 0.927 & 0.692 \\
\bottomrule
  \end{tabular}
  \caption{Sensitivity to the discretization choice ($s = 12/11$ for the one-cycle point).}
  \label{tab:appCalibration_K}
\end{table}

\leadin{Estimator stability.}
Table~\ref{tab:appCalibration_stability}: the slope underlying $L_h$ is insensitive to the
finite-difference step, and every constant is stable across independent calibration-study
draws.

\begin{table}[h]
  \centering
  \small
  \begin{tabular}{ccc}
    \toprule
    perturbation step & $L_h$ & $C(B_\mu)$ \\
    \midrule
    0.005 & 0.7640 & 0.930 \\
    0.010 & 0.7642 & 0.930 \\
    0.020 & 0.7647 & 0.929 \\
    0.040 & 0.7659 & 0.928 \\
\bottomrule
  \end{tabular}
  \hspace{1.5em}
  \begin{tabular}{cccccc}
    \toprule
    seed & $\sigma_h$ & $L_h$ & $B_\mu$ & $\Rmech$ & $C(B_\mu)$ \\
    \midrule
    2024 & 0.3935 & 0.7647 & 0.2211 & 0.417 & 0.929 \\
    2025 & 0.3989 & 0.7643 & 0.2136 & 0.442 & 0.971 \\
    2047 & 0.4181 & 0.7625 & 0.2229 & 0.428 & 0.977 \\
    2101 & 0.4092 & 0.7735 & 0.2276 & 0.407 & 0.928 \\
\bottomrule
  \end{tabular}
  \caption{Left: finite-difference step for the $L_h$ slope. Right: independent
  calibration-study draws ($n = 2{,}000$ each).}
  \label{tab:appCalibration_stability}
\end{table}

\begin{figure}[!htbp]
  \centering
  \includegraphics[width=\linewidth]{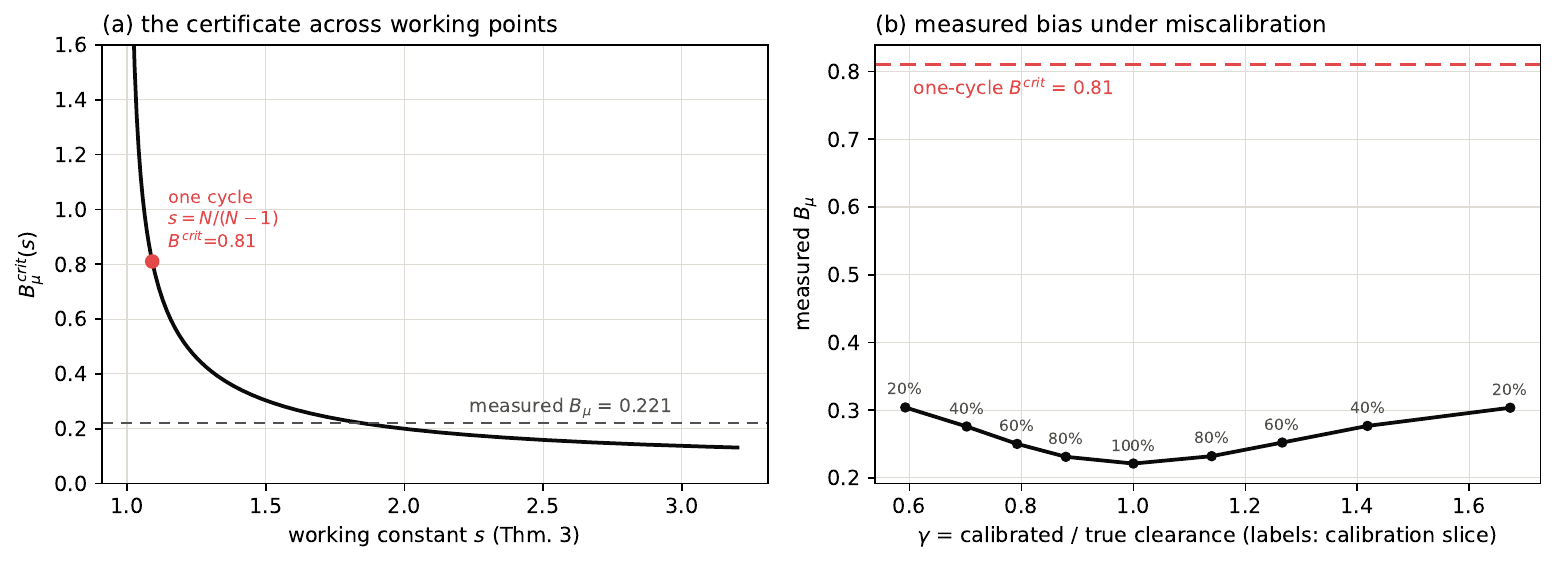}
  \caption{(a) Theorem~\ref{thm:phase}'s critical bias across the working constant $s$,
  with the one-cycle working point marked against the measured $B_\mu$ (grey dashed).
  (b) Measured $B_\mu$ under miscalibration (point labels: the calibration slice), against
  the one-cycle threshold.}
  \label{fig:appCalibration_sensitivity_new}
\end{figure}

\subsection{Design-axis sensitivity of the certificate and the regret ordering}
\label{sec:appCalibration_ofat}

\leadin{Protocol.}
Sections~\ref{sec:appCalibration_policies}--\ref{sec:appCalibration_sensitivity} report one
deployment. This subsection asks which of the trial designer's choices that deployment
depends on, varying each \emph{one at a time} with everything else pinned. Four plant axes
are swept: the number of dose arms $K \in \{4,6,8,12,16\}$ on the fixed
$1{,}600$--$3{,}600$~mg/m$^2$ range, the scored horizon $N \in \{3,6,11,20,50\}$, the
calibration-study size $n_{\mathrm{study}} \in \{200,500,1000,2000,5000\}$ at three study
seeds, and the number of lead-in cycles $\in \{1,2,3\}$: $28$ configurations, each a
complete re-run of the three-cohort protocol of
Section~\ref{sec:appCalibration_algorithms} with $n = 10{,}000$ evaluation patients,
evaluation seed 11 and study seed 2024. The literature-calibrated constants
$\sigma_{\mathrm{assay}}$, $\sigma_{\mathrm{cyc}}$ and $\Omega^2(V_{\max})$ are \emph{not}
swept: they are the plant's declared population magnitudes
(quoted where a published source exists, assumed otherwise).

The operating point is a \emph{biased} deployment, and deliberately so. The mechanistic
model is fitted on the slowest-clearing $40\%$ of the calibration study
($\gamma \approx 0.702$, the slowest-clearing $40\%$ slice of
Table~\ref{tab:appCalibration_hyb_ode}), the regime in which a population model must be
trusted despite drift. The hybrid's prior table $P(\pistar \mid \pihat)$ is fitted on the same slice the deployed model is calibrated on, so the study and
the deployment share one population model. At $\gamma = 1$ the per-cycle mechanistic controller wins on
every axis (Table~\ref{tab:appCalibration_policytable}). That is not the regime the paper
argues about. Every certificate quantity below ($\Rmech$, $\Bmu$, $C(\Bmu)$, $\Bcrit$,
$C_\nu$) is the deployed, biased model's, measured on the \emph{study} cohort, so it is
checkable before an evaluation patient is treated. The five regret columns and the ratios
are the disjoint evaluation cohort. Intervals are $95\%$ confidence intervals on the mean
over that cohort (paired-bootstrap for ratios). On the $n_{\mathrm{study}}$ axis the
reported band is the min--max over the three study seeds, with each run's own interval as
the cap. Nothing is smoothed: every row is one Monte-Carlo run.

One convention must be stated because the lead-in axis exercises it. The implementation
carries the prefactor $d_Z/2$ in $C(\Bmu)$ and divides Theorem~\ref{thm:phase}'s target by
$d_Z$, both linear in $d_Z$, while $\sigma_h$ and $\sigma_\nu$ keep their single-draw
values. Both $C(\Bmu)$ and $\Bcrit$ therefore grow with the number of lead-in draws for
reasons of accounting rather than of information, and the comparable quantity across those
rows is the per-draw ceiling $C(\Bmu)/d_Z$: $0.750$, $0.722$, $0.713$ nats. Nothing at
$d_Z > 1$, certificate or robustness line, should be read without re-deriving the noise
constants for the pooled draw.

\begin{figure}[!htbp]
  \centering
  \includegraphics[width=\textwidth]{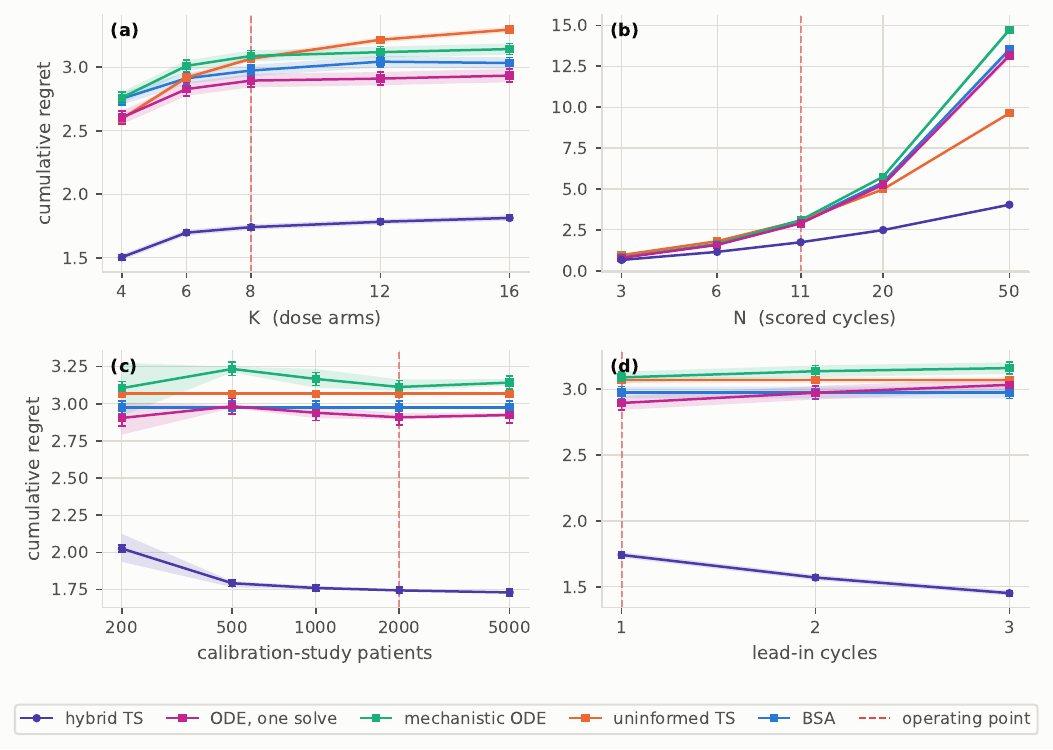}
  \caption{Cumulative regret of the five policies as each design choice is varied alone
  around the operating point: \textbf{(a)} dose arms $K$ on the fixed
  $1{,}600$--$3{,}600$~mg/m$^2$ range, \textbf{(b)} scored horizon $N$, \textbf{(c)}
  calibration-study size, \textbf{(d)} lead-in cycles. Everything not named in a panel is
  held at the operating point ($n = 10{,}000$ evaluation patients,
  $n_{\mathrm{study}} = 2{,}000$, 11 scored cycles, seeds 11/2024). The
  literature-calibrated plant constants are never varied. The deployed mechanistic model
  is \emph{biased}, its population clearance fitted on the slowest-clearing $40\%$ of the
  study ($\gamma = 0.702$), and the hybrid’s prior table is fitted on that same slice. Bands and capped bars are $95\%$ confidence
  intervals on the mean, and on (c), which runs three study seeds per size, the band is
  the min--max over seeds. The dashed rule marks the operating point. The hybrid has the
  lowest regret in all 28 configurations, the ODE solved once on the drifted model is above
  it in all 28 and below the per-cycle controller in all 28, and the per-cycle controller,
  which re-asserts the drifted model every cycle, sits at or above BSA almost everywhere.}
  \label{fig:ofat_policies}
\end{figure}

\begin{figure}[!htbp]
  \centering
  \includegraphics[width=\textwidth]{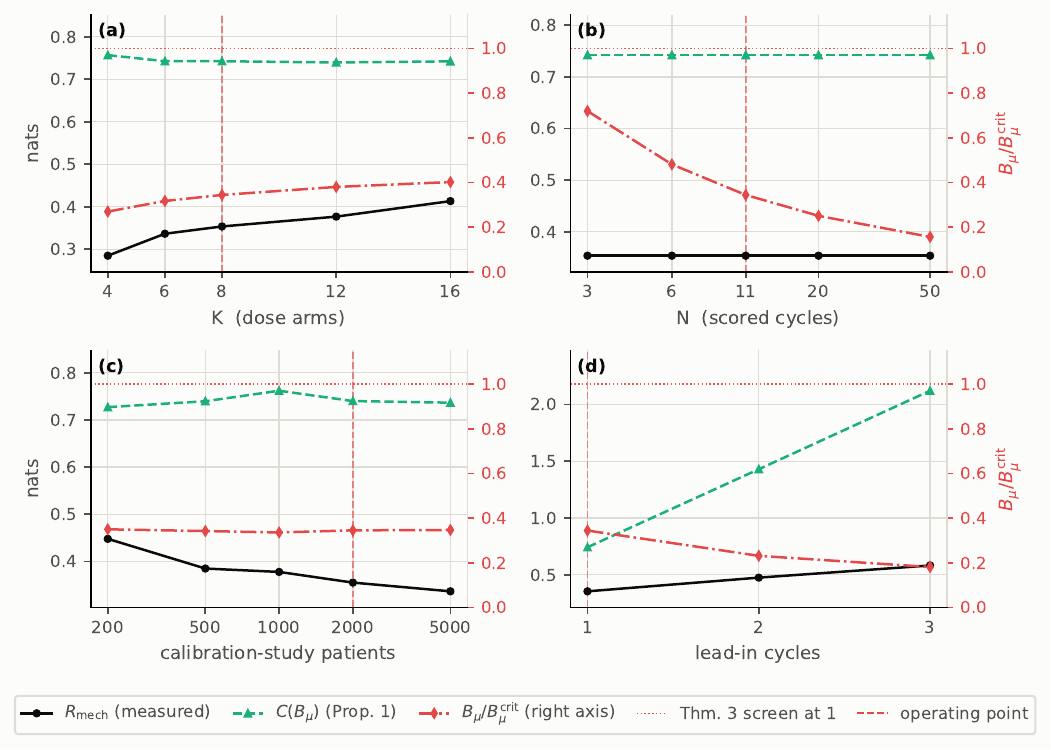}
  \caption{The pre-trial certificate along the same four axes. Left axis, in nats: the
  measured mechanistic information $\Rmech = I(\pistar;\pihat)$ and
  Proposition~\ref{prop:mechinf}'s bias-form ceiling
  $C(\Bmu) = \tfrac{d_Z}{2}\log\!\big(1+\sigma_h^2/(L_h^2\Bmu^2)\big)$. The certificate is
  that the former lies below the latter. Right axis: the measured bias normalised by
  Theorem~\ref{thm:phase}'s threshold, $\Bmu/\Bcrit$, against the dotted screen at 1. The
  theorem is a rule-out, so passing is necessary and not sufficient. All quantities are
  the deployed (biased) model's, measured on the study cohort, hence checkable before any
  evaluation patient is treated. Both statements hold in all configurations. The screen
  is tightest at $N = 3$ (margin $1.40\times$), where saving one cycle in three is the
  most demanding target a model can be set. Panel (d) alone exercises the $d_Z$
  convention: $C(\Bmu)$ and $\Bcrit$ carry a factor linear in $d_Z$ while $\sigma_h$ and
  $\sigma_\nu$ stay single-draw, so the rise of the ceiling there is that scaling. Per
  draw it is $0.750$, $0.722$, $0.713$ nats.}
  \label{fig:ofat_cert}
\end{figure}

\begin{table}[!htbp]
\centering\small\setlength{\tabcolsep}{4pt}
\begin{tabular}{llcccccccccc}
\toprule
axis & value & $\Rmech$ & $C(\Bmu)$ & $\Bmu$ & $\Bcrit$ &
margin & hybrid & once & uninf. & BSA & mech. \\
\midrule
$K$ (arms)          & 4    & 0.290 & 0.765 & 0.279 & 1.046 & 3.75 & 1.504 & 2.604 & 2.595 & 2.754 & 2.761 \\
                    & 6    & 0.343 & 0.750 & 0.276 & 0.879 & 3.19 & 1.698 & 2.829 & 2.919 & 2.913 & 3.012 \\
                    & \textbf{8} & 0.364 & 0.750 & 0.276 & 0.810 & 2.94 & 1.741 & 2.895 & 3.068 & 2.975 & 3.089 \\
                    & 12   & 0.386 & 0.747 & 0.275 & 0.732 & 2.66 & 1.783 & 2.911 & 3.217 & 3.045 & 3.119 \\
                    & 16   & 0.417 & 0.749 & 0.275 & 0.692 & 2.52 & 1.814 & 2.935 & 3.297 & 3.034 & 3.144 \\
\midrule
$N$ (cycles)        & 3    & 0.364 & 0.750 & 0.276 & 0.387 & \textbf{1.40} & 0.661 & 0.790 & 0.955 & 0.811 & 0.804 \\
                    & 6    & 0.364 & 0.750 & 0.276 & 0.580 & 2.10 & 1.153 & 1.579 & 1.807 & 1.623 & 1.644 \\
                    & \textbf{11} & 0.364 & 0.750 & 0.276 & 0.810 & 2.94 & 1.741 & 2.895 & 3.068 & 2.975 & 3.089 \\
                    & 20   & 0.364 & 0.750 & 0.276 & 1.113 & 4.04 & 2.487 & 5.264 & 4.980 & 5.408 & 5.748 \\
                    & 50   & 0.364 & 0.750 & 0.276 & 1.786 & 6.47 & 4.038 & 13.161 & 9.632 & 13.521 & 14.733 \\
\midrule
$n_{\mathrm{study}}$& 200  & 0.416 & 0.746 & 0.287 & 0.838 & 2.92 & 2.017 & 2.912 & 3.068 & 2.975 & 3.122 \\
                    & 500  & 0.370 & 0.766 & 0.287 & 0.869 & 3.03 & 1.806 & 2.971 & 3.068 & 2.975 & 3.212 \\
                    & 1000 & 0.373 & 0.743 & 0.285 & 0.830 & 2.91 & 1.752 & 2.934 & 3.068 & 2.975 & 3.158 \\
                    & \textbf{2000} & 0.364 & 0.750 & 0.276 & 0.810 & 2.94 & 1.741 & 2.895 & 3.068 & 2.975 & 3.089 \\
                    & 5000 & 0.343 & 0.741 & 0.283 & 0.822 & 2.90 & 1.728 & 2.912 & 3.068 & 2.975 & 3.122 \\
\midrule
lead-in             & \textbf{1} & 0.364 & 0.750 & 0.276 & 0.810 & 2.94 & 1.741 & 2.895 & 3.068 & 2.975 & 3.089 \\
                    & 2    & 0.479 & 1.444$^\dagger$ & 0.286 & 1.250$^\dagger$ & 4.37 & 1.570 & 2.973 & 3.068 & 2.975 & 3.136 \\
                    & 3    & 0.585 & 2.140$^\dagger$ & 0.289 & 1.613$^\dagger$ & 5.58 & 1.451 & 3.033 & 3.068 & 2.975 & 3.160 \\
\bottomrule
\end{tabular}
\caption{Design-axis sweep at the biased operating point ($\gamma = 0.702$, bold = the
operating point, repeated once per axis). $\Rmech$, $C(\Bmu)$, $\Bmu$ and $\Bcrit$ are
measured on the study cohort in nats, margin $= \Bcrit/\Bmu$, and the five regret columns
(hybrid, ODE solved once, uninformed TS, BSA, per-cycle mechanistic) are cumulative regret
over the scored cycles on the disjoint evaluation cohort ($n = 10{,}000$), $95\%$ intervals
in the repository output. The certificate holds in every row:
$\Bmu < \Bcrit$ and $\Rmech \leq C(\Bmu)$ in $28/28$, as does $\Rmech \leq C_\nu$. The
$n_{\mathrm{study}}$ rows are study seed 2024. The other two seeds (1729, 8675) are in the
repository output and agree to about $0.05$ from $n_{\mathrm{study}} = 500$ upward.
$^\dagger$ raw values at $d_Z = $ lead-in. Per draw the ceiling is $0.750$, $0.722$,
$0.713$. Theorem~\ref{thm:lb}'s floor $\sqrt{H(\mu)/\Hmech}$ is exceeded in every row, by
$1.25\times$--$2.16\times$.}
\label{tab:ofat}
\end{table}

\leadin{Findings.}
The hybrid has the lowest regret in all 28 configurations (BSA/hybrid
$1.23\times$--$3.35\times$, uninformed/hybrid $1.44\times$--$2.39\times$,
mechanistic/hybrid $1.22\times$--$3.65\times$), and under bias the ordering of
Table~\ref{tab:appCalibration_policytable} reverses: the per-cycle mechanistic controller is
at or above BSA in 26 of 28 configurations, the exceptions being $N = 3$ and one
$n_{\mathrm{study}} = 200$ seed, each within $0.1$ of BSA, against $0.44\times$ the
hybrid's regret at $\gamma = 1$. The ODE solved once on the same drifted model is never
below the hybrid here, in $0$ of $28$ configurations, at $1.19\times$--$3.26\times$ its
regret with the tightest ratio at $N = 3$, and it is below the per-cycle controller in
$28$ of $28$: the short-horizon reversal of Table~\ref{tab:appCalibration_nscaling} does
not survive the bias, because at $N = 3$ the open-loop policy plays a drifted arm three
times while the hybrid's first sampled arm is drawn from a table that already accounts for how often the model's recommendation is right.
Bias and horizon push in opposite directions on the one comparison the main text flags as
close. The $N$ axis shows why: nothing about the model depends
on how long the course is scored, so $\Rmech$, $C(\Bmu)$ and $\Bmu$ are identical in all
five rows and only the target moves, yet the controller grows from $0.804$ to $14.73$
while the hybrid grows to $4.04$. A wrong prior is paid for once. A model re-asserted
every cycle re-asserts the drift. The certificate holds throughout, and
it is tightest exactly where the theory says it should be: at $N = 3$ the one-cycle target
$s = N_{\mathrm{course}}/(N_{\mathrm{course}}-1)$ is at its most demanding and the margin
falls to $1.40\times$, the closest this sweep comes to a screen failure. $\Bmu$ and
$C(\Bmu)$ are near-constant across the whole sweep ($0.270$--$0.288$, $0.702$--$0.791$),
reproducing Table~\ref{tab:appCalibration_K}'s $K$-flatness on a drifted model: the
measurement channel is a property of the model and the assay, not of the design. Two axes
carry a caution. At $n_{\mathrm{study}} = 200$ the plug-in mutual-information estimator
inflates $\Rmech$ by roughly $0.1$ nat over the value the large studies converge to
($0.416$, $0.458$, $0.494$ against $\approx 0.35$), the order of the Miller--Madow
correction $(K-1)^2/2n \approx 0.12$ nats at $n = 200$, $K = 8$, and hybrid regret
degrades to $1.94$--$2.12$: a 200-patient conditional table is about 25 patients per arm.
On the lead-in axis, extra draws help under bias, as they do at the calibrated point ($1.741 \rightarrow 1.451$ against $1.66 \rightarrow 1.24$), even though $\Bmu$ rises rather than falls, $0.276 \rightarrow 0.289$: the information the extra draws add outweight the bias they sharpen.
Finally, Theorem~\ref{thm:lb}'s floor is exceeded in every
configuration, by $1.25\times$ to $2.16\times$, and is met nowhere, as a minimax floor
should be.

\end{document}